\documentclass[final,5p,times,twocolumn]{elsarticle}
\usepackage{amsmath,amssymb,mathtools,bm,amsthm}
\usepackage{booktabs,tabularx,array}
\usepackage{siunitx}
\usepackage{xcolor}
\usepackage{microtype}
\microtypesetup{protrusion=false}
\usepackage{enumitem}
\usepackage{graphicx}
\usepackage{tikz}
\usepackage{pgfplots}
\usepackage{flushend}
\usepackage{needspace}
\usepackage{placeins}
\usepackage{dblfloatfix}
\usepackage{hyperref}

\usetikzlibrary{arrows.meta,calc,positioning}
\pgfplotsset{compat=1.18}
\newcommand{\vect}[1]{\bm{#1}}
\newcommand{\mat}[1]{\bm{#1}}
\newcommand{\R}{\mathbb{R}}
\newcommand{\trans}{^{\mathsf T}}
\newcommand{\norm}[1]{\left\lVert #1\right\rVert}
\newcommand{\abs}[1]{\left\lvert #1\right\rvert}
\newcommand{\diag}{\operatorname{diag}}
\newcommand{\rank}{\operatorname{rank}}
\newcommand{\nullsp}{\operatorname{ker}}

\newcommand{\dd}{\,\mathrm{d}}
\newcommand{\wrap}{\operatorname{wrap}_{[-\pi,\pi)}}
\newtheorem{definition}{Definition}
\newtheorem{assumption}{Assumption}
\newtheorem{lemma}{Lemma}
\newtheorem{remark}{Remark}
\theoremstyle{plain}
\newtheorem{theorem}{Theorem}

\biboptions{numbers,sort&compress}
\hypersetup{hidelinks}
\journal{Mechanism and Machine Theory}
\date{}

\begin{document}

\begin{frontmatter}

\title{Prescribed-Time Contracting-Boundary Control of a Tendon-Driven Flexible Arm}

\author[1]{Yi Lu}
\ead{yilu84007@gmail.com}

\author[2]{Chao Tang}
\ead{tangchao@ouc.edu.cn}

\author[2]{Zhiji Han \corref{cor}}
\ead{hanzhiji@ouc.edu.cn}

\author[2]{Hongdu Wang}
\ead{wanghongdu@ouc.edu.cn}

\cortext[cor]{Corresponding author.}

\affiliation[1]{organization={School of Information Management, Hubei University of Economics},
	addressline={No. 8 Yangqiaohu Road, Canglongdao Development Park, Jiangxia District}, 
	city={Wuhan},
	postcode={430205}, 
	country={China}}

\affiliation[2]{organization={College of Engineering, Ocean University of China},
	addressline={No. 1299 Sansha Road, Qingdao West Coast New Area}, 
	city={Qingdao},
	postcode={266404}, 
	country={China}}

\begin{abstract}
    This study develops a prescribed-time performance-shaping control method for curvature tracking of a single-segment flexible arm actuated by three antagonistic tendon pairs. A Cartesian curvature representation is introduced to avoid the undefined bending direction at the straight configuration and to establish an explicit six-tendon kinematic mapping. A cubic performance boundary contracts smoothly from an initially admissible error bound to a nonzero terminal accuracy bound within a prescribed time. Based on this boundary, a dual transformation combining static symmetric error scaling and time-varying behavior shaping maps the tracking error into a fixed unit box. The resulting controller guarantees boundary invariance, prescribed-time entry into the terminal accuracy region, and subsequent asymptotic convergence. Numerical evaluations with Python and OpenCR--MuJoCo, together with a supervised reduced-order experiment on a two-section, four-channel platform, provide complementary validation. Across six experimental trials, no violation of the prescribed boundary is observed, and the proposed controller reduces the mean terminal curvature RMSE by 32.5\% relative to a matched baseline, with comparable terminal-band entry times. These results support the feasibility of the proposed approach in the reduced-order experimental setting.
\end{abstract}


\end{frontmatter}

\section{Introduction}
\label{sec:introduction}
Tendon-driven flexible arms have demonstrated their capability in a range of applications, particularly where slenderness, compliance, and distal accessibility are important \cite{robinson1999continuum,hannan2003elephant,walker2013continuum}. Representative scenarios include time-critical handling and manipulation, assistive operations for field and space applications, and inspection and
maintenance in confined spaces such as aero-engine passages, pipes, ducts, and equipment cavities \cite{burgnerkahrs2015survey,dupont2022medical,song2026bronchoscope}. 

The modeling and control of such systems remain challenging due to the intrinsic compliance and deformation characteristics of flexible mechanisms
\cite{trivedi2008softrobotics,rus2015softrobots}. Detailed mechanical models have been developed to capture the dynamic behavior of flexible robotic systems
\cite{gravagne2003dynamics,renda2012steadystate,renda2014dynamic,
till2019cosserat,dellasantina2020feedback,katzschmann2019motion,
thuruthel2018control}. The first technical challenge concerns the choice of shape representation. Under the constant curvature assumption, curvature magnitude and bending direction are widely adopted to describe the configuration of flexible arms
\cite{jones2006multisection,jones2006realtime,webster2010review,
dellasantina2020parametrization}. However, when the curvature magnitude vanishes, the bending direction becomes undefined, even though the corresponding straight configuration remains physically regular. The Cartesian curvature vector provides a unique representation of this state with a zero vector, enabling reference generation, estimation, and feedback to pass continuously through the straight configuration without coordinate singularities. The second technical challenge arises from tendon actuation. Tendon routing geometry and coupled actuation effects lead to a nonlinear tendon-to-shape relationship
\cite{camarillo2008mechanics,camarillo2009tracking,rucker2011routing,
rao2021benchmark,yuan2019friction}. In the symmetric architecture considered here, six tendon shortenings correspond to only two bending coordinates, resulting in a redundant actuator space. Tendon friction and tension loss can further decouple the motor-side tendon displacement from the actual deformation of a tendon-driven continuum arm
\cite{liu2021tensionloss}. Consequently, curvature computed from tendon displacement through a kinematic model should not be treated as an independent measurement of the arm shape. This distinction is important for experimental validation, since using tendon-derived curvature as a reference measurement would not provide an independent assessment of the actual shape tracking performance. These two challenges motivate a shape representation that remains regular through the straight configuration and a compatible tendon-level formulation that explicitly accounts for actuation redundancy.

Flexible arms are increasingly expected to operate under simultaneous constraints on task completion time, transient behavior, and terminal accuracy. Delayed arrival may cause a task window to be missed, excessive transient deviation may interfere with surrounding objects or operators, and persistent terminal error may prevent reliable alignment. These requirements have motivated prescribed performance control methods, which shape transient and terminal tracking behavior through performance functions and error transformations
\cite{bechlioulis2008ppc,bechlioulis2009bounds,Huang2019GivenPerformance}.
Related barrier Lyapunov function methods instead provide systematic means for enforcing constant or time-varying bounds on system outputs and states
\cite{tee2009blfoutput,tee2011timevarying,liu2016fullstate,liu2017nussbaum,
liu2018stochastic,he2017manipulator}. For tendon-driven flexible arms, however, achieving such prescribed tracking behavior also requires a suitable shape representation and an explicit mapping from the desired shape evolution to the redundant tendon actuators. These requirements are particularly challenging when the shape representation must remain regular through the straight configuration and the actuator commands must be consistent with the multi-tendon kinematics.

This study develops a prescribed-time contracting-boundary control method for curvature-vector tracking, including continuous passage through the straight configuration. A cubic performance boundary is constructed to contract from an initially admissible error bound to a nonzero terminal accuracy bound within a prescribed time and to remain constant thereafter. Based on this boundary, a dual transformation combines static symmetric error scaling with time-varying behavior shaping to map the curvature-tracking error into a fixed unit box. Under the stated ideal kinematic assumptions and initial feasibility conditions, the proposed method establishes boundary invariance, prescribed-time entry into the terminal accuracy band, and subsequent asymptotic convergence. The main contributions are summarized as follows.
\begin{enumerate}[label=(\arabic*)]
  \item A Cartesian curvature vector representation is introduced to resolve the undefined bending direction at zero curvature in the conventional curvature magnitude and bending direction formulation. Based on this representation, a constant curvature model is established for the six-tendon architecture, together with the mappings between tendon variables and curvature. A reference curvature velocity and a compatible minimum-norm allocation of the six tendon velocities are then derived to realize the desired curvature evolution through the redundant tendon actuation, with all mappings remaining well defined through the straight configuration.
  \item A prescribed-time contracting-boundary control law is proposed. A cubic performance boundary first contracts the admissible curvature-tracking error from an initial bound to a nonzero terminal accuracy bound within a prescribed time and remains constant thereafter. A dual transformation based on static symmetric error scaling and time-varying behavior shaping then maps the time-varying error constraint into a fixed unit box. Under the stated assumptions, boundary invariance, prescribed-time entry into the terminal accuracy band, and subsequent asymptotic convergence are established.
  \item A hierarchical validation framework combining numerical simulations, an application-inspired contact-free branch-alignment scenario, and a supervised reduced-order physical experiment is developed. The hardware evaluation verifies one signed curvature component on an available four-channel platform and is intended as a feasibility study rather than full six-tendon validation.
\end{enumerate}

The remainder of this paper is organized as follows. Section~\ref{sec:preliminaries} introduces the preliminaries for performance shaping, followed by the single-segment six-tendon kinematic model developed in Section~\ref{sec:kinematics}. The proposed prescribed-time contracting-boundary controller and the corresponding closed-loop analysis are presented in Section~\ref{sec:control}. Section~\ref{sec:numerical-simulation} then evaluates the proposed approach through numerical studies and a reduced-order physical experiment. Finally, Section~\ref{sec:conclusions} concludes the paper.

\textbf{Notations:} Bold symbols denote vectors or matrices. The superscript $\mathsf T$ denotes the transpose. $\|\cdot\|_2$ denotes the Euclidean norm for vectors and the induced spectral norm for matrices. For scalars, $\abs{\cdot}$ denotes the absolute value, and $\diag(\cdot)$ forms a diagonal matrix. $\R$ is the real number field, $\mat I_n$ denotes the $n\times n$ identity matrix, and $\mathrm{SO}(2)$ is the group of planar proper rotation matrices. The operators $\operatorname{blkdiag}(\cdot)$, $\operatorname{range}(\cdot)$, $\ker(\cdot)$, and $\operatorname{rank}(\cdot)$ denote block-diagonal construction, matrix range, matrix kernel, and matrix rank, respectively.

\section{Preliminaries}
\label{sec:preliminaries}
This section introduces the technical tools underlying the proposed prescribed-time contracting-boundary controller. First, a cubic performance boundary is constructed to specify the admissible curvature-tracking error and contract it to a nonzero terminal accuracy bound within a prescribed time. The resulting time-varying constraint is then transformed into a fixed unit interval through two successive steps: static symmetric scaling normalizes the error, and time-varying behavior shaping maps the time-varying constraint into a fixed normalized constraint. Finally, a logarithmic barrier Lyapunov function is introduced for the subsequent boundary-invariance analysis.

\begin{definition}[Prescribed-time contracting performance boundary]
    For each $j\in\{x,y\}$, let $e_j(t)$ denote the $j$th component of the curvature tracking error. Define the cubic shaping polynomial
    \begin{equation}
        h(\tau)=1 - 3\tau^2 + 2\tau^3, \quad \tau\in[0,1],
    \end{equation}
    and the prescribed-time performance boundary
    \begin{equation} \label{eq:performance_envelope}
        B_j(t) = 
        \begin{cases}
            \varepsilon_j + (B_{j0} - \varepsilon_j) h(t/T_j), & 0 \le t \le T_j,\\
            \varepsilon_j, & t>T_j.
        \end{cases}
    \end{equation}
    The design parameters satisfy
    \begin{equation} \label{eq:envelope_conditions}
        B_{j0}>\abs{e_j(0)},\qquad
        B_{j0}>\varepsilon_j>0,\qquad
        T_j>0.
    \end{equation}
    The boundary contracts from the initial bound $B_{j0}$ to the terminal bound $\varepsilon_j$ over the prescribed interval $[0,T_j]$ and remains constant thereafter. The associated behavior-shaping multiplier is defined as
    \begin{equation}
        \beta_j(t) = \frac{B_{j0}}{B_j(t)}.
    \end{equation}
\end{definition}

\begin{lemma}[Properties of the prescribed-time contracting boundary] \label{lem:boundary_properties}
    For each $j\in\{x,y\}$, the performance boundary $B_j$ and the associated behavior-shaping multiplier $\beta_j$ are positive and continuously differentiable. The cubic shaping polynomial $h(\tau)$ satisfies
    \begin{equation}
        h(0)=1, \quad h(1)=0,
    \end{equation}
    and is strictly decreasing on $[0,1]$. Consequently,
    \begin{subequations} \label{eq:envelope_endpoints}
        \begin{gather}
            B_j(0)=B_{j0}, \quad B_j(T_j)=\varepsilon_j,\label{eq:envelope_endpoint_values}\\
            \dot B_j(0)=0, \quad \dot B_j(T_j)=0,\label{eq:envelope_endpoint_rates}\\
            \varepsilon_j\le B_j(t)\le B_{j0},\label{eq:envelope_bounds}\\
            B_j(t)=\varepsilon_j,\quad t\ge T_j. \label{eq:envelope_terminal_value}
        \end{gather}
    \end{subequations}
    Moreover, for $0<t<T_j$,
    \begin{subequations} \label{eq:envelope_derivative}
        \begin{gather}
            \dot B_j(t)=-\frac{6(B_{j0}-\varepsilon_j)t(T_j-t)}{T_j^3}<0,\label{eq:envelope_rate}\\
            \max_{0\le t\le T_j}\abs{\dot B_j(t)}=\frac{3(B_{j0}-\varepsilon_j)}{2T_j}.\label{eq:envelope_max_rate}
        \end{gather}
    \end{subequations}
    Hence, $B_j(t)$ contracts strictly from $B_{j0}$ to $\varepsilon_j$ over $[0,T_j]$ and remains constant thereafter. The associated behavior-shaping multiplier satisfies
    \begin{subequations} \label{eq:beta_derivative}
        \begin{gather}
            1\le\beta_j(t)\le\frac{B_{j0}}{\varepsilon_j},\label{eq:beta_bounds}\\
            \dot\beta_j(t)=-\frac{B_{j0}\dot B_j(t)}{B_j^2(t)}\ge0.\label{eq:beta_rate}
        \end{gather}
    \end{subequations}
    Thus, $\beta_j(t)$ is nondecreasing on $[0,T_j]$ and remains constant for $t\ge T_j$.
\end{lemma}

\begin{proof}
    Differentiating the cubic polynomial gives
    \begin{equation} \label{eq:proof_h_derivative}
        \frac{\dd h(\tau)}{\dd \tau} = -6\tau + 6\tau^2 = -6\tau (1-\tau).
    \end{equation}
    Hence, $h'(\tau)<0$ for $\tau\in(0,1)$, while $h'(0) = h'(1)=0$. For $0 < t < T_j$, differentiating Eq.~\eqref{eq:performance_envelope} yields
    \begin{equation} \label{eq:proof_boundary_rate}
        \dot B_j(t) = \frac{B_{j0}-\varepsilon_j}{T_j}h'(t/T_j) = -\frac{6(B_{j0}-\varepsilon_j)t(T_j-t)}{T_j^3}.
    \end{equation}
    Since $B_{j0} > \varepsilon_j > 0$, it follows that $\dot B_j(t) < 0$ on $(0, T_j)$. Moreover, $h(0) = 1$ and $h(1) = 0$ give $B_j(0)=B_{j0}$, $B_j(T_j)=\varepsilon_j$, while $h'(0) = h'(1)=0$ gives $\dot B_j(0) = \dot B_j(T_j)=0$. The strict decrease of $B_j$ on $(0,T_j)$ therefore yields Eq.~\eqref{eq:envelope_bounds}. For $t > T_j$, $B_j(t) = \varepsilon_j$ by definition, so Eq.~\eqref{eq:envelope_terminal_value} follows directly. The right derivative at $T_j$ is zero, while Eq.~\eqref{eq:proof_boundary_rate} gives
    \begin{equation}
        \lim_{t\to T_j^-}\dot B_j(t)=0.
    \end{equation}
    Thus, both $B_j$ and its first derivative are continuous at $T_j$, which establishes $B_j \in C^1$. To determine the maximum contraction rate, define $f(t)=t(T_j-t)$. Then
    \begin{equation} \label{eq:proof_contraction_stationarity}
        \dot{f}(t)=T_j-2t.
    \end{equation}
    The unique stationary point on $[0,T_j]$ is $t = T_j/2$, and since $f(0) = f(T_j) = 0$, its maximum is attained at this point. In particular, \label{eq:proof_contraction_maximum}
    \begin{equation}
        f(T_j/2)=\frac{T_j^2}{4}.
    \end{equation}
    Substitution into Eq.~\eqref{eq:proof_boundary_rate} gives Eq.~\eqref{eq:envelope_max_rate}. Finally, Eq.~\eqref{eq:envelope_bounds} together with $\varepsilon_j > 0$ implies that $B_j(t)$ is strictly positive. Hence $\beta_j(t) = B_{j0}/B_j(t)$ is positive and $C^1$. The bounds on $B_j$ immediately yield Eq.~\eqref{eq:beta_bounds}, while differentiation gives
    \begin{equation}
        \dot\beta_j(t) = -\frac{B_{j0}\dot B_j(t)}{B_j^2(t)} \geq 0,
    \end{equation}
    where the inequality follows from $\dot B_j(t)\leq0$. Thus, Eq.~\eqref{eq:beta_rate} holds, completing the proof.
\end{proof}

\begin{remark}
    The performance boundary $B_j(t)$ is $C^1$ but generally not $C^2$ at $T_j$, owing to its cubic construction on $[0,T_j]$. It contracts from the initial bound $B_{j0}$ to the terminal bound $\varepsilon_j$ over the prescribed interval $[0,T_j]$ and remains constant thereafter. This boundary is adopted in the proposed controller to prescribe the transient contraction and terminal accuracy of the tracking error.
\end{remark}

\begin{definition}[Symmetric dual transformation]
    Let $e_j$ denote the $j$th component of the curvature-tracking error for $j\in\{x,y\}$. The proposed dual transformation consists of two successive mappings:
    \begin{subequations} \label{eq:two_step_transform}
        \begin{align}
            e_j & \xrightarrow{\text{static symmetric scaling}} s_j = \frac{e_j}{B_{j0}}, \label{eq:static_symmetric_scaling} \\
            s_j & \xrightarrow{\text{behavior shaping}} \xi_j = \beta_j(t)s_j = \frac{e_j}{B_j(t)}. \label{eq:behavior_multiplier_scaling}
        \end{align}
    \end{subequations}
    The first mapping symmetrically normalizes the curvature-tracking error, while the second introduces the prescribed time variation through the behavior-shaping multiplier. Hence, the dual transformation maps the time-varying admissible error region into a fixed normalized region. In vector form,
    \begin{subequations} \label{eq:normalized_error}
        \begin{gather}
            \mat D_B(t) = \diag\bigl(B_x(t), B_y(t)\bigr), \label{eq:boundary_matrix} \\
            \vect\xi = 
            \begin{bmatrix}
                \xi_x \\
                \xi_y
            \end{bmatrix}
            = \mat D_B^{-1}(t)
            \begin{bmatrix}
                e_x\\
                e_y
            \end{bmatrix}. \label{eq:vector_normalized_error}
        \end{gather}
    \end{subequations}
\end{definition}

\begin{figure*}[htbp]
    \centering
    \begin{minipage}[b]{0.41\textwidth}
\centering
\begin{tikzpicture}[
  x=1.32cm,
  y=1.32cm,
  >=Latex,
  line cap=round,
  line join=round
]
  \useasboundingbox (-0.20,-0.72) rectangle (3.65,2.68);
  \draw[->,black!55] (-0.12,0)--(3.28,0) node[right] {\(c_x\)};
  \draw[->,black!55] (0,-0.12)--(0,2.53) node[above] {\(c_y\)};
  \draw[dashed,black!40] (0,1.62)--(2.15,1.62)--(2.15,0);
  \draw[very thick,->,blue!72!black] (0,0)--(2.15,1.62)
    node[midway,above left] {\(\vect c\)};
  \draw[->] (0.68,0) arc[start angle=0,end angle=37,radius=0.68];
  \node[font=\scriptsize] at (19:0.94) {\(\alpha\)};
  \node[font=\scriptsize,below] at (2.15,0) {\(c_x\)};
  \node[font=\scriptsize,left] at (0,1.62) {\(c_y\)};

  \fill[black] (0,0) circle (0.045);
  \node[anchor=north west,font=\scriptsize] at (0.10,-0.13)
    {\(\vect c=\vect0\)};
  \node[anchor=north west,font=\scriptsize] at (0.10,-0.42)
    {straight configuration};
\end{tikzpicture}

\smallskip
\textbf{(a)} curvature-vector representation
\end{minipage}
\hfill
\begin{minipage}[b]{0.56\textwidth}
\centering
\begin{tikzpicture}
\begin{axis}[
  width=\linewidth,
  height=4.49cm,
  axis lines=left,
  xmin=0,xmax=1.38,
  ymin=-1.15,ymax=1.15,
  xlabel={time \(t\)},
  ylabel={schematic error/boundary value},
  xtick={0,1},
  xticklabels={\(0\),\(T_j\)},
  ytick={-1,-0.2,0,0.2,1},
  yticklabels={\(-B_{j0}\),\(-\varepsilon_j\),\(0\),\(\varepsilon_j\),\(B_{j0}\)},
  legend style={
    at={(0.98,0.98)},
    anchor=north east,
    draw=none,
    fill=white,
    font=\scriptsize,
    row sep=-2pt
  },
  tick label style={font=\scriptsize},
  label style={font=\scriptsize},
  samples=100,
  clip=false
]
  \addplot[very thick,blue!72!black,domain=0:1]
    {0.2+0.8*(1-3*x^2+2*x^3)};
  \addplot[very thick,blue!72!black,domain=1:1.38] {0.2};
  \addlegendentry{\(+B_j(t)\)}
  \addplot[very thick,blue!72!black,domain=0:1]
    {-0.2-0.8*(1-3*x^2+2*x^3)};
  \addplot[very thick,blue!72!black,domain=1:1.38] {-0.2};
  \addlegendentry{\(-B_j(t)\)}
  \addplot[thick,red!72!black,dashed,domain=0:1.38]
    {0.63*exp(-2.3*x)*cos(deg(4.2*x))};
  \addlegendentry{schematic \(e(t)\)}
  \draw[black!40,thin] (axis cs:0,0)--(axis cs:1.38,0);
  \draw[black!50,densely dotted]
    (axis cs:1,-1.06)--(axis cs:1,1.06);
  \node[anchor=north west,font=\scriptsize] at (axis cs:1.02,-0.34)
    {terminal band};
  \draw[->,thin] (axis cs:1.12,-0.32)--(axis cs:1.12,-0.08);
  \node[
    font=\scriptsize,
    fill=white,
    fill opacity=0.90,
    text opacity=1,
    inner sep=1pt
  ] at (axis cs:0.68,-0.88)
    {\(\xi_j=e_j/B_j,\quad
      \abs{e_j}<B_j\Longleftrightarrow\abs{\xi_j}<1\)};
\end{axis}
\end{tikzpicture}

\smallskip
\textbf{(b)} prescribed performance envelope
\end{minipage}
    \caption{Curvature-vector representation and prescribed-time contracting boundaries: the straight configuration corresponds to the origin, and each symmetric boundary contracts from $B_{j0}$ to the nonzero terminal bound $\varepsilon_j$ within the prescribed time $T_j$. The error trajectory is shown for illustration only.}
    \label{fig:vector_envelope}
\end{figure*}

\begin{lemma}[Transformation equivalence] \label{lem:transformation_equivalence}
    For each fixed $t$, the dual transformation is bijective, with inverse
    \begin{subequations} \label{eq:transformation_equivalence}
        \begin{gather}
            e_j = B_j(t) \xi_j, \label{eq:transformation_inverse} \\
            \abs{\xi_j} < 1 \text{ if and only if } \abs{e_j} < B_j(t). \label{eq:constraint_equivalence}
        \end{gather}
    \end{subequations}
    Moreover, the transformed error satisfies
    \begin{equation}
        \dot{\vect\xi} = \mat D_B^{-1} \left(\dot{\vect e}_c - \dot{\mat D}_B \mat D_B^{-1} \vect e_c\right), \label{eq:xi_derivative}
    \end{equation}
    where $\vect e_c = [e_x,e_y]^{\mathsf T}$.
\end{lemma}

\begin{proof}
    By Eqs.~\eqref{eq:envelope_conditions} and \eqref{eq:envelope_bounds}, $B_j(t) > 0$ for all $t \ge 0$. Hence, for each fixed $t$, the scalar mapping in Eq.~\eqref{eq:behavior_multiplier_scaling} is a linear scaling by the nonzero factor $1/B_j(t)$. Its inverse is given by Eq.~\eqref{eq:transformation_inverse}, which establishes bijectivity. Taking the absolute value of Eq.~\eqref{eq:behavior_multiplier_scaling} gives
    \begin{equation} \label{eq:proof_absolute_normalized_error}
        \abs{\xi_j} = \frac{\abs{e_j}}{B_j(t)}.
    \end{equation}
    Since $B_j(t)>0$, the condition $\abs{\xi_j}<1$ is equivalent to
    \begin{equation}
        \frac{\abs{e_j}}{B_j(t)}<1,
    \end{equation}
    which is equivalent to $\abs{e_j}<B_j(t)$. Thus, Eq.~\eqref{eq:constraint_equivalence} follows. It remains to establish the derivative identity. Differentiating Eq.~\eqref{eq:vector_normalized_error} gives
    \begin{equation} \label{eq:proof_normalized_error_product_rule}
        \dot{\vect\xi} = \frac{\mathrm d}{\mathrm dt}\bigl(\mat D_B^{-1}\bigr)\vect e_c + \mat D_B^{-1} \dot{\vect e}_c.
    \end{equation}
    Differentiating the identity $\mat D_B^{-1} \mat D_B = \mat I_2$ yields
    \begin{equation} \label{eq:proof_inverse_identity_derivative}
        \frac{\mathrm d}{\mathrm dt} \bigl(\mat D_B^{-1}\bigr) \mat D_B + \mat D_B^{-1} \dot{\mat D}_B = \mat 0.
    \end{equation}
    Multiplying this equation on the right by $\mat D_B^{-1}$ gives
    \begin{equation} \label{eq:proof_inverse_matrix_derivative}
        \frac{\mathrm d}{\mathrm dt} \bigl(\mat D_B^{-1}\bigr) = -\mat D_B^{-1}\dot{\mat D}_B\mat D_B^{-1}.
    \end{equation}
    Substitution into Eq.~\eqref{eq:proof_normalized_error_product_rule} then yields
    \begin{equation} \label{eq:proof_xi_derivative_result}
        \dot{\vect\xi} = \mat D_B^{-1}(t) \left(\dot{\vect e}_c - \dot{\mat D}_B(t)\mat D_B^{-1}(t) \vect e_c\right),
    \end{equation}
    which coincides with Eq.~\eqref{eq:xi_derivative}. Therefore, all the claimed properties hold.
\end{proof}

The controller is designed to satisfy the following componentwise performance specifications:
\begin{subequations} \label{eq:prescribed_specifications}
    \begin{gather}
        \abs{e_j(t)} < B_j(t), \label{eq:prescribed_boundary_specification}\\
        \abs{e_j(t)} < \varepsilon_j,\quad t\ge T_j, \label{eq:prescribed_terminal_specification}\\
        \lim_{t\to\infty}e_j(t)=0. \label{eq:prescribed_asymptotic_specification}
    \end{gather}
\end{subequations}
The first specification requires the curvature-tracking error to remain inside the prescribed-time contracting boundary $B_j(t)$. By construction, this boundary contracts from $B_{j0}$ to the nonzero terminal bound $\varepsilon_j$ by the prescribed time $T_j$ and remains constant thereafter. The second specification therefore guarantees that the error enters the terminal accuracy band by $T_j$, while the third establishes subsequent asymptotic convergence to the zero-error state. Figure~\ref{fig:vector_envelope} illustrates the curvature-vector coordinates and the corresponding componentwise contracting boundaries.

\begin{definition}[Logarithmic barrier Lyapunov function] \label{def:barrier_function}
    On the open unit box
    \begin{equation}
        \mathcal X=\{\vect\xi:\abs{\xi_x}<1,\ \abs{\xi_y}<1\},
    \end{equation}
    define the logarithmic barrier Lyapunov function
    \begin{equation} \label{eq:barrier_function}
        V_B(\vect\xi)=\frac{1}{2}\sum_{j\in\{x,y\}}\ln\left(\frac{1}{1-\xi_j^2}\right).
    \end{equation}
    Then $V_B(\vect\xi)\ge0$ for all $\vect\xi\in\mathcal X$, with $V_B(\vect\xi)=0$ if and only if $\vect\xi=\vect0$. Moreover, $V_B(\vect\xi)\to+\infty$ as $\abs{\xi_j}\to1$ for either $j\in\{x,y\}$. Thus, the logarithmic barrier prevents the normalized error from reaching the boundary of the admissible unit box.
\end{definition}

\section{Kinematic modeling} \label{sec:kinematics}
This section develops the kinematic model of the single-segment six-tendon-driven flexible arm. We first introduce the mechanical architecture, Cartesian curvature coordinates, and modeling assumptions, followed by the tendon-to-curvature mapping together with its inverse and differential forms. A continuously extended constant-curvature tip-position map is then established, including the straight configuration.

\subsection{Mechanical architecture and coordinates}
\begin{figure*}[htbp]
    \centering
    \resizebox{\textwidth}{!}{%
\providecommand{\FigureTwoRoot}{figures}
\begin{tikzpicture}[
  x=1cm,
  y=1cm,
  >=Latex,
  line cap=round,
  line join=round,
  paneltitle/.style={font=\footnotesize},
  annotation/.style={font=\scriptsize,fill=white,inner sep=1.2pt}
]
\useasboundingbox (0,-0.48) rectangle (18.8,4.72);

\begin{scope}[shift={(0,0)}]
  \input{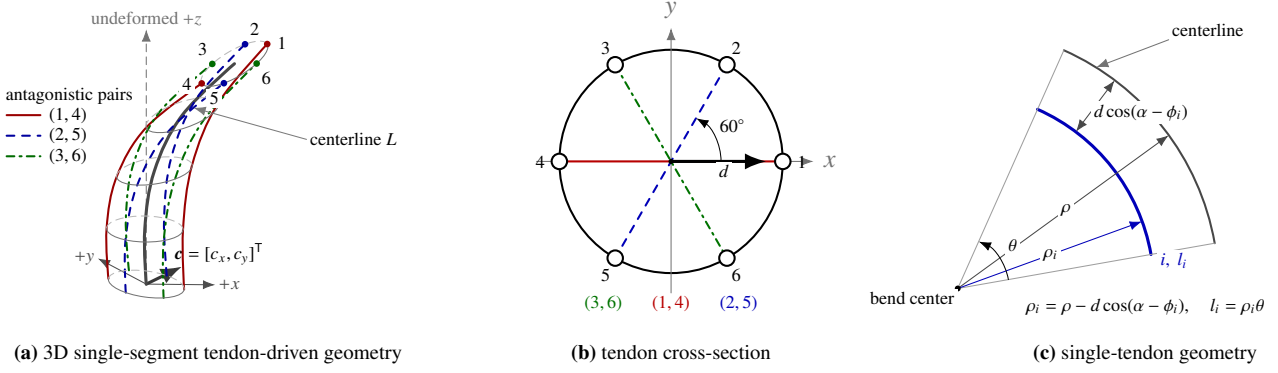}
\end{scope}

\begin{scope}[shift={(6.20,0)}]
  \coordinate (C) at (3.00,2.45);
  \draw[thick] (C) circle (1.52);
  \draw[->,black!55] (1.17,2.45)--(4.94,2.45) node[right] {\(x\)};
  \draw[->,black!55] (3.00,0.66)--(3.00,4.27) node[above] {\(y\)};
  \draw[red!72!black,thick]
    ($(C)+(0:1.52)$)--($(C)+(180:1.52)$);
  \draw[blue!72!black,thick,dashed]
    ($(C)+(60:1.52)$)--($(C)+(240:1.52)$);
  \draw[green!45!black,thick,dash dot]
    ($(C)+(120:1.52)$)--($(C)+(300:1.52)$);
  \foreach \ang/\num in {0/1,60/2,120/3,180/4,240/5,300/6}{
    \fill[white] ($(C)+(\ang:1.52)$) circle (0.105);
    \draw[thick] ($(C)+(\ang:1.52)$) circle (0.105);
    \node[font=\scriptsize] at ($(C)+(\ang:1.78)$) {\num};
  }
  \fill (C) circle (0.035);
  \draw[very thick,->] (C)--($(C)+(0:1.30)$);
  \node[annotation,below] at ($(C)+(0:0.72)$) {\(d\)};
  \draw[->] ($(C)+(0:0.68)$)
    arc[start angle=0,end angle=60,radius=0.68];
  \node[annotation] at ($(C)+(30:0.97)$) {\(60^\circ\)};
  \node[font=\scriptsize,green!45!black] at (2.08,0.52) {\((3,6)\)};
  \node[font=\scriptsize,red!72!black] at (3.00,0.52) {\((1,4)\)};
  \node[font=\scriptsize,blue!72!black] at (3.92,0.52) {\((2,5)\)};

  \node[paneltitle,anchor=base] at (3.00,-0.24)
    {\textbf{(b)} tendon cross-section};
\end{scope}

\begin{scope}[shift={(12.55,0)}]
  \coordinate (O) at (0.55,0.72);
  \def\startang{10}
  \def\endang{66}
  \def\rhoc{3.55}
  \def\rhot{2.67}

  \draw[thick,black!72]
    ($(O)+(\startang:\rhoc)$)
    arc[start angle=\startang,end angle=\endang,radius=\rhoc];
  \draw[very thick,blue!72!black]
    ($(O)+(\startang:\rhot)$)
    arc[start angle=\startang,end angle=\endang,radius=\rhot];
  \draw[black!38] (O)--($(O)+(\startang:\rhoc)$);
  \draw[black!38] (O)--($(O)+(\endang:\rhoc)$);

  \draw[->,black!72] (O)--($(O)+(36:\rhoc)$);
  \node[annotation] at ($(O)+(36:1.82)$) {\(\rho\)};
  \draw[->,blue!72!black] (O)--($(O)+(20:\rhot)$);
  \node[annotation] at ($(O)+(20:1.35)$) {\(\rho_i\)};
  \draw[<->,black!72]
    ($(O)+(53:\rhot)$)--($(O)+(53:\rhoc)$);
  \node[annotation,anchor=west] at ($(O)+(53:3.00)$)
    {\(d\cos(\alpha-\phi_i)\)};
  \draw[->] ($(O)+(10:0.70)$)
    arc[start angle=10,end angle=66,radius=0.70];
  \node[annotation] at ($(O)+(38:0.95)$) {\(\theta\)};
  \fill (O) circle (0.04);
  \node[annotation,below left] at (O) {bend center};

  \node[annotation,anchor=west] (clabel) at (3.42,4.24)
    {centerline};
  \draw[->,black!55] (clabel.south west)--($(O)+(58:\rhoc)$);
  \node[annotation,anchor=west,text=blue!72!black,fill=white] at (3.26,1.10)
    {\(i,\ l_i\)};
  \node[font=\scriptsize,anchor=base] at (3.10,0.42)
    {\(\rho_i=\rho-d\cos(\alpha-\phi_i),\quad l_i=\rho_i\theta\)};

  \node[paneltitle,anchor=base] at (3.10,-0.24)
    {\textbf{(c)} single-tendon geometry};
\end{scope}
\end{tikzpicture}%
}
    \caption{Idealized geometry of the single-segment tendon-driven arm: (a) three-dimensional structure with six tendon paths, (b) tendon cross-section showing three antagonistic tendon pairs, and (c) constant-curvature geometry for tendon $i$.}
    \label{fig:combined_geometry}
\end{figure*}

Consider a single flexible segment with centerline length $L>0$. Six tendons are routed through guides located at a constant radial offset $d>0$ from the centerline and are arranged into three antagonistic pairs: $(1,4)$, $(2,5)$, and $(3,6)$. The $z$-axis of the base frame is aligned with the centerline of the arm in its straight configuration, and the positive $x$-axis passes through tendon~1. The tendon angular positions are given by
\begin{equation} \label{eq:tendon_angles}
    \phi_i=(i-1)\frac{\pi}{3}, i\in\{1,\ldots,6\}.
\end{equation}
When defined, the bending direction $\alpha$ is measured counterclockwise from the positive $x$-axis toward the positive $y$-axis. Figure~\ref{fig:combined_geometry} illustrates the tendon numbering and the corresponding ideal constant-curvature configuration.

Each tendon is modeled as a tension-transmission element, and its differential shortening is measured relative to a calibrated zero configuration. Pretension is treated separately from the shape coordinates and therefore cannot be determined from the differential tendon shortening $\vect q$. The tendon numbering and motor rotation directions are calibrated separately to maintain consistency between the actuator commands and the tendon coordinates. Let $l_i$ and $l_{i,\mathrm{ref}}$ denote the actual length and calibrated reference length of tendon $i$, respectively. The differential shortening is defined as
\begin{equation} \label{eq:q_sign}
    q_i=l_{i,\mathrm{ref}}-l_i,
\end{equation}
so that $q_i>0$ denotes tendon shortening. In the ideal kinematic model, $l_{i,\mathrm{ref}}=L$, whereas an experimental platform may require
individual calibrated reference lengths for the six tendons. Define
\begin{equation}
    \vect l=[l_1,\ldots,l_6]^{\mathsf T}, \quad \vect q=[q_1,\ldots,q_6]^{\mathsf T}.
\end{equation}
Since the reference lengths are constant, the tendon and shortening velocities satisfy
\begin{equation}
    \dot{\vect l}=-\dot{\vect q}.
\end{equation}
This sign convention is retained when mapping tendon velocities to the motor commands. The Cartesian curvature vector provides a polar-to-Cartesian reparameterization of the conventional curvature magnitude and bending direction:
\begin{equation}
    \vect c=
    \begin{bmatrix}
        c_x\\
        c_y
    \end{bmatrix}
    =
    \begin{bmatrix}
        \kappa\cos\alpha\\
        \kappa\sin\alpha
    \end{bmatrix},
    \label{eq:curvature_vector}
\end{equation}
so that $\kappa=\|\vect c\|$. This representation remains well defined at $\kappa=0$, where $\vect c=\vect 0$ although the bending direction $\alpha$ is undefined. Consequently, the controller can pass through the straight configuration without a change of coordinates or division by $\kappa$. Here, $\vect q=\vect0$ denotes zero differential tendon shortening and does not imply zero tendon tension. For $\kappa>0$, the bending direction is recovered as
\begin{equation}
    \alpha=\operatorname{atan2}(c_y,c_x).
\end{equation}

Let the desired curvature vector be $\vect c_d=[c_{d,x},c_{d,y}]^{\mathsf T}$, and define the curvature-tracking error as
\begin{equation}
    \vect e_c=
    \begin{bmatrix}
        e_x\\
        e_y
    \end{bmatrix}
    =\vect c_d-\vect c.
    \label{eq:error_definition}
\end{equation}
The desired-minus-actual sign convention in Eq.~\eqref{eq:error_definition} is used throughout this work. For each $j\in\{x,y\}$, the componentwise performance requirement is
\begin{equation}
    \abs{e_j(t)}<B_j(t),\qquad t\ge0,
    \label{eq:performance_objective}
\end{equation}
where $B_j(t)$ is the prescribed-time contracting performance boundary. In addition, the controller is required to achieve
\begin{subequations} \label{eq:terminal_objectives}
    \begin{gather}
        \abs{e_j(t)}<\varepsilon_j,\quad t\ge T_j, \\
        \lim_{t\to\infty}\norm{\vect e_c(t)}=0.
    \end{gather}
\end{subequations}
The first condition specifies prescribed-time entry into a nonzero terminal accuracy band, whereas the second specifies asymptotic convergence to the zero-error state under the ideal model. The componentwise contracting boundaries define the admissible error
rectangle in the $(e_x,e_y)$ plane and, consequently,
\begin{equation}
    \norm{\vect e_c(t)}
    <\sqrt{B_x^2(t)+B_y^2(t)}.
\end{equation}
After $T_c=\max\{T_x,T_y\}$, both curvature-error components lie within their respective terminal accuracy bounds.

To establish the prescribed-time tracking properties of the proposed controller, the following assumptions define the nominal kinematic and
execution conditions under which the analytical results are derived.

\begin{assumption}
\label{ass:bending}
    Over the considered workspace, the flexible segment is approximated by a single constant-curvature arc. Axial extension, shear, and torsion are neglected relative to bending. The segment length $L$, tendon offset $d$, and guide angles are assumed to remain constant during each run.
\end{assumption}

Assumption~\ref{ass:bending} adopts a single-segment constant-curvature approximation for the considered bending-dominant workspace
\cite{jones2006multisection,webster2010review,rao2021benchmark}. Effects such as axial extension, shear, torsion, routing friction, and
significant external loading can reduce the accuracy of this approximation \cite{rucker2011routing,yuan2019friction}. Accordingly, the proposed method focuses on curvature regulation rather than axial positioning. Tasks requiring substantial axial motion may be combined with an external linear positioning module, as in related steering systems \cite{song2026bronchoscope}. The shape/model residual reported in Section~\ref{sec:numerical-simulation} provides a numerical check of the approximation over the tested workspace, but does not establish its validity beyond that workspace.

\begin{assumption}
\label{ass:loading}
    Gravity and persistent external loads are either compensated or sufficiently small within the nominal analytical workspace. Contact with the environment is excluded.
\end{assumption}

Assumption~\ref{ass:loading} specifies the nominal load condition for the analytical results. In practice, supports, counterweights, and suitable orientation can reduce gravity-induced deformation, while low-speed operation can further reduce dynamic effects, but cannot eliminate it in general. More comprehensive distributed-load and dynamic models are required when such effects are significant \cite{renda2014dynamic,till2019cosserat}. The gravity-enabled simulations in Section~\ref{sec:numerical-simulation} illustrate the resulting increase in model residual.

\begin{assumption}
\label{ass:execution}
    The desired curvature $\vect c_d(t)$ is continuously differentiable with bounded value and derivative. The commanded tendon velocities are assumed to be executed exactly without saturation, slack, delay, or measurement error.
\end{assumption}

Assumption~\ref{ass:execution} is imposed solely for the analytical development and does not represent a hardware guarantee. The Python
implementation verifies the ideal kinematic and control calculations, while the OpenCR--MuJoCo simulations examine the effects of saturation, slack, delay, and model mismatch. The physical experiment provides reduced-order experimental evidence rather than full six-tendon validation. If $L$, $d$, or $\phi_i$ vary with time, additional parameter-rate terms arise when differentiating the tendon-to-curvature mapping and are not included in the ideal model.

\subsection{Tendon-to-curvature mapping}
For $\kappa>0$, the centerline radius and bending angle are given by
\begin{gather}
    \rho=\frac{1}{\kappa}, \\
    \theta=\kappa L.
\end{gather}
The radius of the circular arc associated with tendon $i$ is
\begin{equation}
    \rho_i=\rho-d\cos(\alpha-\phi_i).
\end{equation}
Hence, its length and differential shortening are
\begin{align}
    l_i
    &=\rho_i\theta
     =L-dL\kappa\cos(\alpha-\phi_i), \\
    q_i
    &=L-l_i
     =dL\kappa\cos(\alpha-\phi_i).
    \label{eq:single_tendon_geometry}
\end{align}
The signed quantity $q_i$ indicates the tendon length change relative to the calibrated reference configuration: $q_i>0$ denotes shortening, whereas $q_i<0$ denotes lengthening. For example, when the arm bends toward $+x$,
\begin{equation}
    q_1=dL\kappa>0,
    \quad
    q_4=-dL\kappa<0.
\end{equation}

Using Eqs.~\eqref{eq:curvature_vector} and \eqref{eq:single_tendon_geometry}, the tendon shortening can be expressed directly in terms of the Cartesian curvature components as
\begin{equation}
    q_i
    =dL\left(c_x\cos\phi_i+c_y\sin\phi_i\right).
    \label{eq:single_tendon_vector}
\end{equation}
Unlike the curvature-magnitude and bending-direction representation, this form remains well defined at the straight configuration $\vect c=\vect0$. Stacking the six tendon relations yields
\begin{equation}
    \vect q=dL\mat C\vect c,
    \quad
    \mat C=
    \begin{bmatrix}
        1&0\\
        \tfrac12&\tfrac{\sqrt3}{2}\\
        -\tfrac12&\tfrac{\sqrt3}{2}\\
        -1&0\\
        -\tfrac12&-\tfrac{\sqrt3}{2}\\
        \tfrac12&-\tfrac{\sqrt3}{2}
    \end{bmatrix}.
    \label{eq:six_tendon_mapping}
\end{equation}
Each row of $\mat C$ represents the radial direction of the corresponding tendon in the transverse plane. The two columns of $\mat C$ describe the tendon shortening patterns associated with the $c_x$- and $c_y$-components of the curvature vector, respectively. Direct calculation gives
\begin{equation}
    \rank(\mat C)=2,
    \quad
    \mat C^{\mathsf T}\mat C=3\mat I_2.
    \label{eq:C_properties}
\end{equation}
The dimensions are consistent: $\vect c$ has units of $\mathrm{m}^{-1}$, $dL$ has units of $\mathrm{m}^2$, and $\mat C$ is dimensionless, yielding $\vect q$ in $\mathrm{m}$. Moreover, the identity $\mat C^{\mathsf T}\mat C=3\mat I_2$ provides the closed-form inverse mapping given below. For the ideal $60^\circ$ tendon arrangement, the two curvature directions are isotropic in the tendon-to-curvature mapping, meaning that the mapping has the same Euclidean gain for bending in the $x$- and $y$-directions. In particular,
\begin{equation}
    \norm{\vect q}_2=\sqrt{3}\,dL\norm{\vect c}_2,
\end{equation}
so the tendon shortening magnitude depends only on the curvature magnitude and is independent of the bending direction. Equivalently, the inverse mapping has the Euclidean gain $1/(\sqrt{3}\,dL)$. In a physical platform, unequal tendon offsets, routing angles, or transmission gains generally destroy this symmetry. The resulting tendon-to-curvature map should therefore be calibrated from independent shape measurements and verified to have full column rank. The detailed derivation of $\mat C$ is provided in \ref{sec:supp_mapping}.

Opposite tendons satisfy
\begin{subequations} \label{eq:compatibility_relations}
    \begin{gather}
    q_4=-q_1,\quad
    q_5=-q_2,\quad
    q_6=-q_3, \\
    q_1-q_2+q_3=0.
    \end{gather}
\end{subequations}
These relations characterize the two-dimensional tendon-shortening subspace compatible with the ideal pure-bending kinematics. Since $\mat C\in\mathbb{R}^{6\times2}$ has rank two, the tendon-to-curvature mapping contains four redundant tendon directions. In particular, the four-dimensional null space of $\mat C^{\mathsf T}$ corresponds to tendon shortening patterns that are invisible to the curvature reconstruction in Eq.~\eqref{eq:six_tendon_mapping}. These null-space motions do not alter the curvature predicted by the ideal kinematic model and therefore do not constitute additional shape coordinates. Accordingly, curvature measurements alone cannot distinguish such null-space tendon states or recover tendon tension. A full basis for the null space and the associated observability discussion are provided in \ref{sec:supp_null}.

\subsection{Inverse and differential kinematics}
Equation~\eqref{eq:C_properties} yields the exact inverse mapping on the ideal pure-bending subspace. For constant $L$, $d$, and $\mat C$, the corresponding differential relations are
\begin{equation}
    \vect c
    =\frac{1}{3dL}\mat C^{\mathsf T}\vect q,
    \quad
    \dot{\vect q}
    =dL\mat C\dot{\vect c},
    \quad
    \dot{\vect c}
    =\frac{1}{3dL}\mat C^{\mathsf T}\dot{\vect q}.
    \label{eq:inverse_differential}
\end{equation}
When the measured tendon-shortening vector does not lie exactly in the ideal pure-bending subspace, the same reconstruction provides the least-squares curvature estimate. For a measured tendon-shortening vector $\vect q_m$, the reconstructed curvature is given by
\begin{equation}
    \hat{\vect c}
    =\frac{1}{3dL}\mat C^{\mathsf T}\vect q_m,
    \label{eq:curvature_reconstruction}
\end{equation}
which is the least-squares solution when $\vect q_m$ does not lie exactly in the ideal pure-bending subspace. The corresponding incompatibility residual is
\begin{equation}
    \vect r_q
    =
    \left(
        \mat I_6-\frac{1}{3}\mat C\mat C^{\mathsf T}
    \right)\vect q_m.
    \label{eq:compatibility_residual}
\end{equation}
A nonzero residual indicates deviation from the ideal pure-bending compatibility relations and may arise from reference-length zeroing errors, calibration errors, tendon slack, or geometric mismatch. The residual alone, however, does not identify the underlying physical cause. The tendon-based curvature reconstruction and the residual $\vect r_q$ assess consistency with the same two-dimensional kinematic model. They do not provide independent measurements of the backbone shape and therefore cannot detect errors that remain within the ideal tendon-to-curvature subspace. Independent shape sensing is consequently used for curvature feedback in the model-based study. The corresponding decomposition into curvature-representable and curvature-invisible tendon components, together with the minimum-norm tendon-velocity allocation, is detailed in \ref{sec:supp_null}.

\subsection{Tip-position model and straight-configuration extension}
\label{sec:tip-position-model}
Let $\kappa=\norm{\vect c}$. For $\kappa>0$, the tip position of the constant-curvature segment is given by
\begin{equation}
    \vect p(\vect c)=
    \begin{bmatrix}
        c_x A(\kappa)\\
        c_y A(\kappa)\\
        G(\kappa)
    \end{bmatrix},
    \label{eq:tip_position}
\end{equation}
where
\begin{align}
    A(\kappa)&=\frac{1-\cos(L\kappa)}{\kappa^2}, \\
    G(\kappa)&=\frac{\sin(L\kappa)}{\kappa}.
\end{align}
The transverse tip position is directed along the curvature vector $\vect c$, while the axial tip position is determined solely by the curvature magnitude $\kappa$. This expression represents the standard constant-curvature kinematics directly in Cartesian curvature coordinates, without introducing a separate bending-plane rotation. Define the tip-position Jacobian as
\begin{equation}
    \mat J_p(\vect c)=
    \frac{\partial \vect p(\vect c)}{\partial \vect c}.
    \label{eq:tip_position_jacobian}
\end{equation}
Since $\vect p$ depends on $\vect c$ both explicitly and implicitly through $\kappa=\norm{\vect c}$, the Jacobian is obtained by the chain rule. For $\kappa>0$,
\begin{equation}
    \frac{\partial\kappa}{\partial\vect c}
    =
    \frac{\vect c^{\mathsf T}}{\kappa},
    \quad
    \frac{\partial\kappa}{\partial c_x}
    =\frac{c_x}{\kappa},
    \quad
    \frac{\partial\kappa}{\partial c_y}
    =\frac{c_y}{\kappa}.
    \label{eq:curvature_magnitude_derivatives}
\end{equation}
The scalar derivatives of $A(\kappa)$ and $G(\kappa)$ are
\begin{align}
    A'(\kappa)
    &=
    \frac{L\kappa\sin(L\kappa)-2[1-\cos(L\kappa)]}{\kappa^3},
    \label{eq:Aprime}\\
    G'(\kappa)
    &=
    \frac{L\kappa\cos(L\kappa)-\sin(L\kappa)}{\kappa^2}.
    \label{eq:Gprime}
\end{align}
Consequently, for $\kappa>0$, the $3\times2$ tip-position Jacobian is
\begin{equation}
    \mat J_p(\vect c)=
    \begin{bmatrix}
        A(\kappa)+\dfrac{A'(\kappa)}{\kappa}c_x^2
        &
        \dfrac{A'(\kappa)}{\kappa}c_xc_y
        \\[6pt]
        \dfrac{A'(\kappa)}{\kappa}c_xc_y
        &
        A(\kappa)+\dfrac{A'(\kappa)}{\kappa}c_y^2
        \\[6pt]
        \dfrac{G'(\kappa)}{\kappa}c_x
        &
        \dfrac{G'(\kappa)}{\kappa}c_y
    \end{bmatrix}.
    \label{eq:tip_position_jacobian_explicit}
\end{equation}
Although the expressions above contain terms involving $1/\kappa$, the apparent singularity at $\kappa=0$ is removable. To evaluate the straight configuration, consider the Taylor expansions
\begin{subequations}
    \label{eq:tip_position_series}
    \begin{align}
        A(\kappa)
        &=
        \frac{L^2}{2}
        -\frac{L^4}{24}\kappa^2
        +\frac{L^6}{720}\kappa^4
        +O(\kappa^6),
        \label{eq:A_series}\\
        G(\kappa)
        &=
        L-\frac{L^3}{6}\kappa^2
        +\frac{L^5}{120}\kappa^4
        +O(\kappa^6), \label{eq:G_series}\\
        A'(\kappa)
        &=
        -\frac{L^4}{12}\kappa
        +\frac{L^6}{180}\kappa^3
        +O(\kappa^5), \label{eq:Aprime_series}\\
        G'(\kappa)
        &=
        -\frac{L^3}{3}\kappa
        +\frac{L^5}{30}\kappa^3
        +O(\kappa^5).
        \label{eq:Gprime_series}
    \end{align}
\end{subequations}
For numerical implementation, the Taylor series expansions are used when $L\kappa$ is sufficiently small to avoid numerical cancellation. This numerical treatment does not alter the analytical tip-position map. Using $\kappa^2=c_x^2+c_y^2$, the corresponding multivariate expansions of the tip-position components are
\begin{subequations}
    \label{eq:tip_position_multivariate_series}
    \begin{align}
        p_x
        &=
        \frac{L^2}{2}c_x
        -\frac{L^4}{24}c_x(c_x^2+c_y^2)
        +O(\norm{\vect c}^5),
        \label{eq:px_series}\\
        p_y
        &=
        \frac{L^2}{2}c_y
        -\frac{L^4}{24}c_y(c_x^2+c_y^2)
        +O(\norm{\vect c}^5),
        \label{eq:py_series}\\
        p_z
        &=
        L-\frac{L^3}{6}(c_x^2+c_y^2)
        +\frac{L^5}{120}(c_x^2+c_y^2)^2
        +O(\norm{\vect c}^6).
        \label{eq:pz_series}
    \end{align}
\end{subequations}
Differentiating these expansions gives
\begin{subequations}
    \label{eq:tip_jacobian_series}
    \begin{align}
        \frac{\partial p_x}{\partial c_x}
        &=
        \frac{L^2}{2}
        -\frac{L^4}{24}(3c_x^2+c_y^2)
        +O(\norm{\vect c}^4),\\
        \frac{\partial p_x}{\partial c_y}
        &=
        -\frac{L^4}{12}c_xc_y
        +O(\norm{\vect c}^4),\\
        \frac{\partial p_y}{\partial c_x}
        &=
        -\frac{L^4}{12}c_xc_y
        +O(\norm{\vect c}^4),\\
        \frac{\partial p_y}{\partial c_y}
        &=
        \frac{L^2}{2}
        -\frac{L^4}{24}(c_x^2+3c_y^2)
        +O(\norm{\vect c}^4),\\
        \frac{\partial p_z}{\partial c_x}
        &=
        -\frac{L^3}{3}c_x
        +O(\norm{\vect c}^3),\\
        \frac{\partial p_z}{\partial c_y}
        &=
        -\frac{L^3}{3}c_y
        +O(\norm{\vect c}^3).
    \end{align}
\end{subequations}
Therefore, the removable straight-state limit is
\begin{equation}
    \vect p(\vect0)=
    \begin{bmatrix}
        0\\
        0\\
        L
    \end{bmatrix},
    \quad
    \mat J_p(\vect0)=
    \begin{bmatrix}
        L^2/2&0\\
        0&L^2/2\\
        0&0
    \end{bmatrix}.
    \label{eq:straight_configuration_jacobian}
\end{equation}
Hence, the tip-position map and its Jacobian are well defined and continuously differentiable through the straight configuration, despite the apparent $1/\kappa$ singularities in the nonzero-curvature representation. For any planar rotation $\mat R\in\mathrm{SO}(2)$, define its three-dimensional extension as $\widetilde{\mat R}=\operatorname{blkdiag}(\mat R,1)$. The tip-position map is rotation equivariant, namely,
\begin{equation}
  \vect p(\mat R\vect c)
  =\widetilde{\mat R}\vect p(\vect c),
  \quad
  \mat J_p(\mat R\vect c)
  =\widetilde{\mat R}\mat J_p(\vect c)\mat R\trans.
  \label{eq:rotation_equivariance}
\end{equation}
Here, a planar rotation of the curvature vector changes only its bending direction while preserving its magnitude, i.e., $\|\mat R\vect c\|=\|\vect c\|$. Accordingly, the horizontal components of the tip position rotate by the same angle, whereas the axial component remains unchanged. Thus, the tip-position map is rotationally equivariant, as expressed in \eqref{eq:rotation_equivariance}. Moreover, since both $\mat R$ and $\widetilde{\mat R}$ are orthogonal, the spectral norm of the Jacobian is invariant under this transformation. 
\begin{equation}
  \norm{\mat J_p(\mat R\vect c)}_2
  =
  \norm{\mat J_p(\vect c)}_2.
  \label{eq:jacobian_rotation_invariance}
\end{equation}
Consequently, any two curvature vectors with the same magnitude but different bending directions yield Jacobians with identical spectral norms. Therefore, over the circular curvature domain $\|\vect c\|\leq\kappa_{\max}$, the maximum Jacobian norm depends only on the curvature magnitude and can be obtained through a one-dimensional search over $\kappa\in[0,\kappa_{\max}]$. 

\section{Prescribed-time contracting-boundary control} \label{sec:control}
This section develops the prescribed-time contracting-boundary controller. A reference curvature velocity is first derived from the transformed-error dynamics and then mapped to the six tendon velocities through a compatible minimum-norm allocation. The resulting closed-loop system is subsequently analyzed in terms of envelope invariance, prescribed-time entry into the terminal accuracy region, and asymptotic convergence.

\begin{figure*}[htbp]
\centering
\resizebox{0.98\textwidth}{!}{%
\begin{tikzpicture}[
  x=1cm,
  y=1cm,
  line cap=round,
  line join=round,
  block/.style={
    draw=black!68,
    line width=0.55pt,
    rounded corners=2pt,
    align=center,
    minimum height=1.06cm,
    inner xsep=3pt,
    inner ysep=2pt,
    font=\footnotesize
  },
  referenceblock/.style={block,draw=blue!45!black,fill=blue!4},
  errorblock/.style={block,draw=blue!45!black,fill=blue!4},
  controllerblock/.style={block,draw=blue!45!black,fill=blue!4},
  allocationblock/.style={block,fill=black!2},
  plantblock/.style={block,draw=orange!45!black,fill=orange!7},
  supportblock/.style={
    block,
    draw=green!45!black,
    fill=green!7,
    minimum height=0.92cm
  },
  diagnosticblock/.style={
    block,
    draw=black!48,
    fill=black!1,
    minimum height=0.72cm,
    font=\scriptsize
  },
  terminalblock/.style={
    diagnosticblock,
    dashed,
    fill=white,
    font=\scriptsize\itshape
  },
  flow/.style={
    -{Latex[length=2.1mm,width=1.35mm]},
    draw=black!72,
    line width=0.70pt,
    shorten <=1.2pt,
    shorten >=1.8pt
  },
  diagnostic/.style={
    flow,
    dashed,
    draw=black!52
  },
  signal/.style={
    font=\scriptsize,
    text=black!82,
    fill=white,
    inner xsep=1.5pt,
    inner ysep=0.5pt
  }
]
\useasboundingbox (0,-3.78) rectangle (18.50,1.28);

\node[referenceblock,text width=2.15cm] (reference) at (1.25,0)
  {reference generator\\%
   \(\vect c_d,\dot{\vect c}_d\)};
\node[errorblock,text width=2.60cm] (error) at (4.45,0)
  {feedback-error computation\\[-0.1ex]
   \(\hat{\vect e}_c=\vect c_d-\hat{\vect c}_{\rm shape}\)};
\node[controllerblock,text width=3.00cm] (controller) at (8.10,0)
  {dual-transformation\\performance-shaping controller\\[-0.15ex]
   \(\scriptstyle B_j,\dot B_j,\mat K\)};
\node[allocationblock,text width=2.45cm] (allocation) at (11.65,0)
  {six-tendon velocity allocation\\[-0.1ex]
   \(dL\mat C\)};
\node[plantblock,text width=2.65cm] (plant) at (15.10,0)
  {OpenCR--MuJoCo\\flexible-arm model};

\draw[flow] (reference.east)
  -- node[signal,above] {\(\vect c_d\)} (error.west);
\draw[flow] (error.east)
  -- node[signal,above] {\(\hat{\vect e}_c\)} (controller.west);
\draw[flow] (controller.east)
  -- node[signal,above] {\(\dot{\vect c}_r\)} (allocation.west);
\draw[flow] (allocation.east)
  -- node[signal,above] {\(\dot{\vect q}\)} (plant.west);

\coordinate (ffstart) at ([xshift=4mm]reference.north);
\coordinate (fftop) at ($(ffstart)+(0,0.45)$);
\coordinate (ffin) at ([xshift=-5mm]controller.north);
\coordinate (ffcorner) at (fftop -| ffin);
\draw[flow,rounded corners=4pt]
  (ffstart) -- (fftop)
  -- node[signal,above] {\(\dot{\vect c}_d\)} (ffcorner)
  -- (ffin);

\node[supportblock,text width=2.50cm] (shapefit) at (15.10,-1.68)
  {constant-curvature shape fit\\[-0.1ex]\(\hat{\vect c}_{\rm shape}\)};
\draw[flow] (plant.south)
  -- node[signal,right] {backbone nodes} (shapefit.north);
\coordinate (feedbackcorner) at (4.45,-1.68);
\draw[flow,rounded corners=3pt]
  (shapefit.west)
  -- node[signal,above,pos=0.46] {\(\hat{\vect c}_{\rm shape}\)}
  (feedbackcorner)
  -- (error.south);

\node[diagnosticblock,text width=2.05cm] (states) at (9.10,-3.28)
  {measured\\tendon states\\\(\vect q_m\)};
\node[diagnosticblock,text width=2.15cm] (estimate) at (5.55,-3.28)
  {tendon-derived estimate\\%
   \(\hat{\vect c}_{q}\)};
\node[terminalblock,text width=2.10cm] (diagnostics) at (2.15,-3.28)
  {diagnostics only};
\coordinate (diagout) at ($(plant.east)+(0.38,0)$);
\coordinate (diagcorner) at (17.10,-3.28);
\draw[diagnostic,rounded corners=3pt]
  (plant.east) -- (diagout) |- (diagcorner) -- (states.east);
\draw[diagnostic] (states.west) -- (estimate.east);
\draw[diagnostic] (estimate.west) -- (diagnostics.east);
\end{tikzpicture}%
}
\caption{Implemented kinematic control architecture. The solid loop uses
constant-curvature shape-fitted curvature feedback, whereas the tendon-derived curvature
estimate is retained only for diagnostics.}
\label{fig:control_architecture}
\end{figure*}
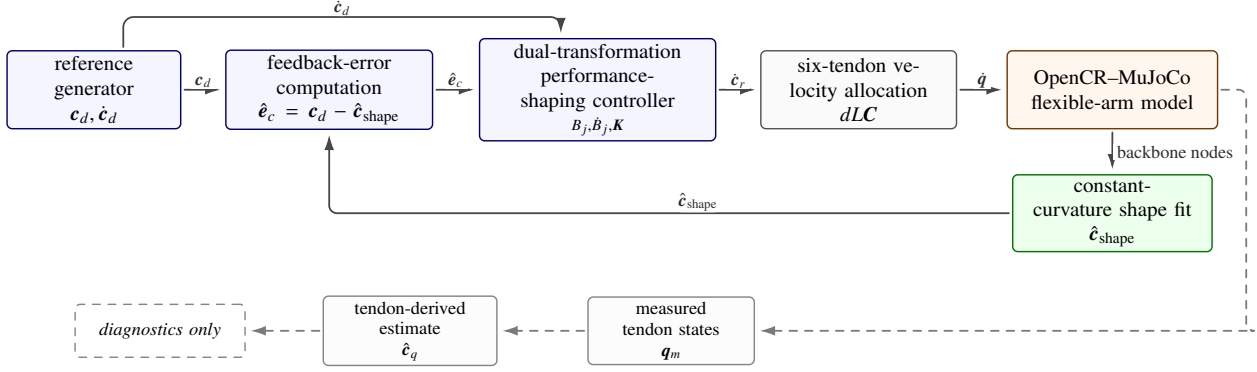

\subsection{Reference curvature velocity}
The control objective is to satisfy the componentwise performance envelope and terminal conditions in Eqs.~\eqref{eq:performance_objective}--\eqref{eq:terminal_objectives} for any initially feasible tracking error. The preferred command interface is the Cartesian curvature vector $\vect c_d(t)$. If the desired curvature is specified by its magnitude $\kappa_d(t)$ and bending direction $\alpha_d(t)$, rather than directly by the Cartesian curvature vector, its Cartesian components and their time derivatives are obtained as
\begin{equation} \label{eq:desired_components}
    c_{d,x}=\kappa_d\cos\alpha_d, \quad c_{d,y}=\kappa_d\sin\alpha_d,
\end{equation}
and
\begin{subequations} \label{eq:desired_component_rates}
    \begin{gather}
        \dot c_{d,x}
        =\dot\kappa_d\cos\alpha_d
        -\kappa_d\dot\alpha_d\sin\alpha_d, \\
        \dot c_{d,y}
        =\dot\kappa_d\sin\alpha_d
        +\kappa_d\dot\alpha_d\cos\alpha_d.
    \end{gather}
\end{subequations}
When $\kappa_d=0$, the Cartesian vector representation avoids assigning an arbitrary bending direction $\alpha_d$. Accordingly, commands that pass through the straight configuration are specified directly in Cartesian curvature coordinates. The desired curvature vector and its time derivative should be generated
consistently in Cartesian coordinates to preserve continuity through the straight configuration and avoid artificial discontinuities in the feedforward term. To achieve asymptotic convergence while preserving the transformed error within the unit box, the desired dynamics of the normalized tracking error are specified as
\begin{equation}
  \dot{\vect\xi}=-\mat K\vect\xi,
  \quad
  \mat K=\diag(k_x,k_y),
  \quad
  k_x,k_y>0.
  \label{eq:desired_xi_dynamics}
\end{equation}
Solving Eq.~\eqref{eq:xi_derivative} together with
\(\dot{\vect e}_c=\dot{\vect c}_d-\dot{\vect c}\) yields the reference
curvature velocity
\begin{equation}
  \dot{\vect c}_r=
  \dot{\vect c}_d
  -\dot{\mat D}_B\mat D_B^{-1}\vect e_c
  +\mat K\vect e_c.
  \label{eq:reference_curvature_velocity}
\end{equation}
The three terms represent the desired-rate feedforward, boundary-contraction compensation, and curvature-error feedback, respectively. In particular, the $j$th component of the boundary-contraction term is
$-\dot B_j e_j/B_j$, which reinforces the error-correcting action because $\dot B_j\le0$ during the contraction phase. The diagonal matrix $\mat K$ has units of inverse time. Its diagonal structure also allows it to commute with the scaling matrix $\mat D_B$. Since $\dot B_j(T_j)=0$ and $B_j(t)$ remains constant for $t\ge T_j$, the boundary-contraction term vanishes after $T_j$. Consequently, the controller reduces to the desired-rate feedforward plus proportional curvature-error feedback once the terminal performance boundaries have been reached. Componentwise, Eq.~\eqref{eq:reference_curvature_velocity} becomes
\begin{equation}
  \dot c_{r,j}
  =
  \dot c_{d,j}
  -\frac{\dot B_j}{B_j}e_j
  +k_je_j,
  \quad j\in\{x,y\}.
  \label{eq:componentwise_reference_rate}
\end{equation}
Since $B_j(t)\ge\varepsilon_j>0$, $\dot B_j(t)$ is bounded, and $\abs{e_j(t)}<B_j(t)$, the boundary-contraction term $-\dot B_j e_j/B_j$ remains bounded. Hence, a bounded desired curvature rate $\dot{\vect c}_d$ implies a bounded reference curvature rate $\dot{\vect c}_r$. This analytical boundedness result does not by itself guarantee that the resulting tendon velocity commands satisfy the physical actuator limits.

Using the maximum boundary-contraction rate from Eq.~\eqref{eq:envelope_max_rate}, and assuming $\abs{\dot c_{d,j}}\le\bar v_{d,j}$, boundary invariance together with $\abs{e_j}<B_j\le B_{j0}$ yields the conservative componentwise bound
\begin{equation}
    \abs{\dot c_{r,j}}
    <
    \bar v_{d,j}
    +\frac{3(B_{j0}-\varepsilon_j)}{2T_j}
    +k_jB_{j0}.
\end{equation}
After mapping the calibrated tendon-speed limits to corresponding curvature-rate limits, this bound can be used to screen the parameter set $(B_{j0},\varepsilon_j,T_j,k_j)$. The resulting executed trajectory should
then be checked against sampled-data and reel-transmission constraints.

\subsection{Six-tendon velocity allocation and implementation}
The compatible tendon-velocity command is given by
\begin{equation}
  \dot{\vect q}=dL\mat C\dot{\vect c}_r.
  \label{eq:tendon_velocity_command}
\end{equation}
Using Eq.~\eqref{eq:C_properties}, its Euclidean norm satisfies
\begin{equation}
    \norm{\dot{\vect q}}_2
  =\sqrt3\,dL\norm{\dot{\vect c}_r}_2,
  \quad
  \abs{\dot q_i}
  \le dL\norm{\dot{\vect c}_r}_2.
\end{equation}
The first relation characterizes the total tendon-speed magnitude induced by the reference curvature rate, whereas the second provides a componentwise upper bound because each row of $\mat C$ has unit Euclidean norm. Once the tendon- and motor-speed limits are known, these relations provide necessary feasibility checks for the selected performance envelopes and desired trajectory.

The proposed allocation is the unique minimum-Euclidean-norm solution to the stated curvature-velocity matching problem. More generally, all tendon velocities that produce the same curvature velocity can be written as
\begin{equation}
  \dot{\vect q}
  =
  dL\mat C\dot{\vect c}_r
  +
  \left(
    \mat I_6-\frac{1}{3}\mat C\mat C^{\mathsf T}
  \right)\vect w,
  \quad
  \vect w\in\mathbb{R}^6.
  \label{eq:general_allocation}
\end{equation}
The proposed controller selects $\vect w=\vect0$, thereby eliminating the null-space component. A tension-aware allocator could instead exploit this null-space freedom while preserving the same curvature velocity.
However, the present allocation does not explicitly regulate tendon tension, prevent slack, or account for actuator saturation. The two terms in Eq.~\eqref{eq:general_allocation} belong to the orthogonal subspaces $\operatorname{range}(\mat C)$ and $\ker(\mat C^{\mathsf T})$, respectively. Consequently, their squared
Euclidean norms add, and any nonzero null-space component strictly increases the tendon-velocity norm. Hence, the choice $\vect w=\vect0$ yields the unique minimum-Euclidean-norm allocation. This is a minimum-velocity-norm result and does not imply minimum energy consumption or minimum tendon force.

The proposed command preserves the ideal antagonistic and compatibility relations,
\begin{equation}
    \dot q_4=-\dot q_1,\quad
    \dot q_5=-\dot q_2,\quad
    \dot q_6=-\dot q_3,\quad
    \dot q_1-\dot q_2+\dot q_3=0.
\end{equation}
These relations provide convenient online consistency checks. However, deviations from them may result from measurement or modeling errors and should not be interpreted as direct evidence of tendon tension.

If a calibrated actuation matrix is used in place of $dL\mat C$, the same matrix and sign convention must be used consistently for both curvature reconstruction and tendon-velocity allocation. Otherwise, a systematic mismatch may arise in the closed loop even when the calibrated matrix is full column rank. Independent shape measurements, rather than algebraic self-consistency alone, are therefore required to validate the calibrated kinematic model. For the implemented signals, $\hat{\vect c}_{\mathrm{shape}}$ is obtained by fitting a single-segment constant-curvature model to the measured backbone nodes and is used as the feedback signal. The measured tendon shortening is denoted by $\vect q_m$, from which $\hat{\vect c}_q=(3dL)^{-1}\mat C^{\mathsf T}\vect q_m$ is computed solely for model-based diagnostics. Figure~\ref{fig:control_architecture}
visually distinguishes these two signal paths.

At each control update in the OpenCR--MuJoCo implementation described in Section~\ref{sec:opencr-mujoco}, the following procedure is executed:
\begin{enumerate}[label=(\arabic*)]
  \item fit a single-segment constant-curvature centerline to the measured backbone nodes to obtain $\hat{\vect c}_{\mathrm{shape}}$;
  \item evaluate the desired curvature, performance envelopes, and shape-based tracking error
  $\hat{\vect e}_c=\vect c_d-\hat{\vect c}_{\mathrm{shape}}$;
  \item compute the reference curvature velocity from Eq.~\eqref{eq:reference_curvature_velocity} and the compatible tendon velocity command from Eq.~\eqref{eq:tendon_velocity_command}, followed by command integration and actuator-limit enforcement;
  \item record the shape-based normalized error $\hat{\vect\xi}$, actuator saturation, tendon tension, and the model-based diagnostic estimate $\hat{\vect c}_q=(3dL)^{-1}\mat C^{\mathsf T}\vect q_m$.
\end{enumerate}

The scalar sampled-data update used in the reduced-order hardware experiment is described separately in Section~\ref{sec:physical-experiments}. For the hardware implementation, the tendon-length rates satisfy
$\dot{\vect l}=-\dot{\vect q}$ and must be converted to motor commands using the calibrated reel radii and motor sign conventions. Individual command clipping can destroy the compatible six-tendon velocity pattern, while sampling, computational delay, and zero-order holding introduce deviations from the continuous-time identity $\dot{\vect c}=\dot{\vect c}_r$. The applied commands and saturation events are therefore recorded explicitly for subsequent analysis. Reel conversion and the extended implementation details are provided in \ref{sec:supp_commands}.

\subsection{Closed-loop performance analysis}
\label{sec:analysis}
Under the ideal actuation assumptions, substituting Eq.~\eqref{eq:tendon_velocity_command} into
Eq.~\eqref{eq:inverse_differential} and using Eq.~\eqref{eq:C_properties} yields the exact curvature-velocity relation
\begin{equation}
  \dot{\vect c}
  =
  \frac{1}{3dL}\mat C^{\mathsf T}
  \left(dL\mat C\dot{\vect c}_r\right)
  =\dot{\vect c}_r.
  \label{eq:exact_curvature_velocity}
\end{equation}
Since $\vect e_c=\vect c_d-\vect c$, its time derivative is $\dot{\vect e}_c=\dot{\vect c}_d-\dot{\vect c}$. Substituting Eq.~\eqref{eq:reference_curvature_velocity} and Eq.~\eqref{eq:exact_curvature_velocity} then yields
\begin{equation}
  \dot{\vect e}_c
  =
  \dot{\mat D}_B\mat D_B^{-1}\vect e_c
  -\mat K\vect e_c.
  \label{eq:error_dynamics}
\end{equation}
Substituting Eq.~\eqref{eq:error_dynamics} into Eq.~\eqref{eq:xi_derivative}, and using the commutativity of the diagonal matrices $\mat K$, $\mat D_B$, and $\dot{\mat D}_B$, gives
\begin{equation}
  \dot{\vect\xi}=-\mat K\vect\xi.
  \label{eq:closed_loop_dynamics}
\end{equation}
Equations~\eqref{eq:exact_curvature_velocity}--\eqref{eq:closed_loop_dynamics} establish the nominal normalized-error dynamics underlying the subsequent analysis.

\begin{theorem}[Closed-loop prescribed-performance properties] \label{thm:closed_loop_properties}
    Suppose Assumptions~\ref{ass:bending}--\ref{ass:execution} hold. For each $j\in\{x,y\}$, let $B_j(t)$ satisfy Eqs.~\eqref{eq:performance_envelope}--\eqref{eq:envelope_conditions}, let $\mat K=\diag(k_x,k_y)$ with $k_x,k_y>0$, and apply the controller in~\eqref{eq:reference_curvature_velocity} together with the compatible tendon-velocity allocation in~\eqref{eq:tendon_velocity_command}. Since $B_j(0)=B_{j0}$, the initial-feasibility condition in Eq.~\eqref{eq:envelope_conditions} is equivalent to $\abs{\xi_j(0)}<1$. Then, under the ideal closed-loop dynamics, the following properties hold:
    \begin{enumerate}[label=(\arabic*)]
        \item \emph{Boundary invariance:}
        \begin{equation}
            \abs{\xi_j(t)}<1,
            \quad
            \abs{e_j(t)}<B_j(t),
            \quad \forall\,t\ge0.
        \end{equation}
        \item \emph{Prescribed-time terminal-band entry:}
        \begin{equation}
            \abs{e_j(t)}<\varepsilon_j,
            \quad \forall\,t\ge T_j.
        \end{equation}
        In particular, with $T_c=\max\{T_x,T_y\}$, for all $t\ge T_c$,
        \begin{equation}
            \abs{e_x(t)}<\varepsilon_x,
            \quad
            \abs{e_y(t)}<\varepsilon_y,
            \quad
            \norm{\vect e_c(t)}
            <\sqrt{\varepsilon_x^2+\varepsilon_y^2}.
        \end{equation}
        \item \emph{Asymptotic convergence:}
        \begin{equation}
            \lim_{t\to\infty}\vect\xi(t)=\vect0,
            \quad
            \lim_{t\to\infty}\vect e_c(t)=\vect0.
        \end{equation}
    \end{enumerate}
\end{theorem}

\begin{proof}
    Since $\mat K$ is diagonal, Eq.~\eqref{eq:closed_loop_dynamics} decouples componentwise as $\dot{\xi}_j=-k_j\xi_j$. Hence,
    \begin{subequations} \label{eq:xi_solution}
        \begin{align}
            \vect\xi(t)&=\exp(-\mat Kt)\vect\xi(0),\\
            \xi_j(t)&=\xi_j(0)\exp(-k_jt).
        \end{align}
    \end{subequations}
    Because $k_j>0$, $\exp(-k_jt)\le1$ for all $t\ge0$. Thus,
    \begin{equation}
        \abs{\xi_j(t)}
        =\abs{\xi_j(0)}\exp(-k_jt)
        \le\abs{\xi_j(0)}<1.
    \end{equation}
    By Lemma~\ref{lem:transformation_equivalence}, $e_j(t)=B_j(t)\xi_j(t)$, and by Lemma~\ref{lem:boundary_properties}, $B_j(t)>0$. It follows that
    \begin{equation}
        \abs{e_j(t)}
        =B_j(t)\abs{\xi_j(t)}
        <B_j(t),
        \quad t\ge0,
    \end{equation}
    which establishes boundary invariance under the initially feasible condition. For $t\ge T_j$, Lemma~\ref{lem:boundary_properties} gives $B_j(t)=\varepsilon_j$. Therefore,
    \begin{equation}
        \abs{e_j(t)}
        =B_j(t)\abs{\xi_j(t)}
        =\varepsilon_j\abs{\xi_j(t)}
        <\varepsilon_j,
        \quad t\ge T_j.
    \end{equation}
    Hence, each error component enters and remains within its prescribed nonzero terminal band by time $T_j$. In particular, for $t\ge T_c=\max\{T_x,T_y\}$, both componentwise bounds hold simultaneously, and therefore
    \begin{equation}
        \norm{\vect e_c(t)}^2
        =e_x^2(t)+e_y^2(t)
        <\varepsilon_x^2+\varepsilon_y^2.
    \end{equation}
    Thus,
    \begin{equation}
        \norm{\vect e_c(t)}
        <\sqrt{\varepsilon_x^2+\varepsilon_y^2},
        \quad t\ge T_c.
    \end{equation}
    This establishes prescribed-time entry into a nonzero terminal accuracy region, rather than finite-time convergence to the origin. Finally, Eq.~\eqref{eq:xi_solution} and \(k_j>0\) imply
    \begin{equation}
        \lim_{t\to\infty}\xi_j(t)=0.
    \end{equation}
    Since $0<B_j(t)\le B_{j0}$ and $e_j(t)=B_j(t)\xi_j(t)$,
    \begin{equation}
        \abs{e_j(t)}
        \le B_{j0}\abs{\xi_j(t)}
        \longrightarrow0.
    \end{equation}
    Thus, both components converge to zero, yielding
    \begin{equation}
        \lim_{t\to\infty}\vect\xi(t)=\vect0,
        \quad
        \lim_{t\to\infty}\vect e_c(t)=\vect0.
    \end{equation}
\end{proof}

Theorem~\ref{thm:closed_loop_properties} follows directly from the prescribed normalized-error dynamics. The barrier function introduced in Definition~\ref{def:barrier_function} provides a complementary interpretation of boundary invariance. It is nonnegative within the open square $\abs{\xi_j}<1$ and diverges as either component approaches its boundary. Along Eq.~\eqref{eq:closed_loop_dynamics},
\begin{equation}
  \dot V_B=
  -\sum_{j\in\{x,y\}}
  \frac{k_j\xi_j^2}{1-\xi_j^2}
  \le0.
  \label{eq:barrier_derivative}
\end{equation}
Hence, a finite initial barrier value cannot evolve to the infinite boundary, and $\dot V_B=0$ only when $\vect\xi=\vect0$. Moreover, using $-\ln(1-s)\ge s$ for $s\in[0,1)$ gives
\begin{equation}
    V_B\ge\frac{1}{2}\norm{\vect\xi}^2,
\end{equation}
which shows that every finite sublevel set of $V_B$ is bounded. Since $V_B$ is continuous on the open square, these sublevel sets are closed. Furthermore, $V_B(\vect\xi)\to\infty$ as any component approaches $\pm1$, so every finite sublevel set remains strictly inside the open
square and is therefore compact. In particular, if $V_B(t)\le V_B(0)$, then each logarithmic term is bounded by $2V_B(0)$, yielding
\begin{equation}
    \abs{\xi_j(t)}
    \le
    \sqrt{1-\exp[-2V_B(0)]}
    <1.
\end{equation}
This conservative estimate makes the separation from the prescribed boundary explicit.

\begin{remark}
    The condition $\abs{\xi_j(0)}<1$ is simply the normalized form of the initial-feasibility condition in Eq.~\eqref{eq:envelope_conditions} and does not constitute an additional assumption. Since no recovery mechanism is provided for states on or outside the prescribed boundary, initialization uncertainty should be accounted for when selecting $B_{j0}$. The terminal result guarantees entry into a nonzero accuracy region by $T_j$, rather than finite-time convergence to the origin. In particular, the ideal error satisfies
    \begin{equation}
        e_j(t)=B_j(t)\xi_j(0)\exp(-k_jt).
        \label{eq:error_solution}
    \end{equation}
    Because $B_j(t)>0$ is nonincreasing, the sign of any nonzero ideal error is preserved and its magnitude is nonincreasing. Thus, under the ideal model, the symmetric implementation produces no overshoot across zero, although it does not provide independently tunable positive and negative error bounds. For a time-varying reference, these properties remain conditional on exact desired-rate feedforward and ideal velocity execution. The prescribed time $T_j$ determines the terminal-band entry time, whereas $k_j$ determines the normalized-error decay rate. Neither parameter eliminates the actuator-rate constraints. The margin $m_j=1-\abs{\xi_j}$ is nondecreasing under the ideal dynamics and can therefore serve as a useful diagnostic quantity in physical experiments, where it should be considered together with actuator saturation and measurement uncertainty. Finally, because the boundary is positive and $C^1$, the normalized solution and the resulting curvature and compatible tendon velocity remain finite on every finite time interval under the ideal model, provided that $\dot{\vect c}_d$ is bounded.
\end{remark}

\begin{remark}
    If the actual execution deviates from the reference curvature velocity according to $\dot{\vect c}=\dot{\vect c}_r+\vect\delta_c$, the normalized-error dynamics become
    \begin{equation}
        \dot{\vect\xi}
        =-\mat K\vect\xi-\mat D_B^{-1}\vect\delta_c.
    \end{equation}
    Suppose that a platform-specific bound $\abs{\delta_{c,j}}\le\bar\delta_j$ is available. Evaluating the resulting scalar dynamics at the boundary points $\xi_j=\pm1$ yields the sufficient inward-pointing condition
    \begin{equation}
        \bar\delta_j<k_jB_j(t).
    \end{equation}
    A conservative time-uniform condition is therefore
    \begin{equation}
        \bar\delta_j<k_j\varepsilon_j.
    \end{equation}
    No experimentally verified platform-specific bound $\bar\delta_j$ is established in the present study. Accordingly, Theorem~\ref{thm:closed_loop_properties} provides nominal prescribed-performance guarantees and does not establish robust boundary invariance. Effects such as actuator saturation, friction, execution delay, tendon slack, and model mismatch are instead treated as practical considerations in Section~\ref{sec:numerical-simulation}.
\end{remark}

\begin{remark}
    The Cartesian curvature-vector error provides a unified error measure for both curvature magnitude and bending direction. In particular, let $\kappa_d=\norm{\vect c_d}$. By the reverse triangle inequality,
    \begin{equation}
        \abs{\kappa_d-\kappa}
        \le\norm{\vect e_c}.
        \label{eq:magnitude_error_bound}
    \end{equation}
    Hence, convergence of the Cartesian curvature-vector error directly implies convergence of the curvature-magnitude error. For the bending-direction error, assume $\kappa,\kappa_d\ge\kappa_{\min}>0$, and define
    \begin{equation}
        e_\alpha=\wrap(\alpha_d-\alpha),
        \quad
        \theta=\abs{e_\alpha}\in[0,\pi],
    \end{equation}
    where $\wrap(\cdot)\in[-\pi,\pi)$ denotes the principal-angle representative. By the law of cosines,
    \begin{align}
    \norm{\vect e_c}^2
    &=\kappa_d^2+\kappa^2-2\kappa_d\kappa\cos\theta\nonumber\\
    &=(\kappa_d-\kappa)^2
    +2\kappa_d\kappa(1-\cos\theta)\nonumber\\
    &=(\kappa_d-\kappa)^2
    +4\kappa_d\kappa\sin^2(\theta/2)\nonumber\\
    &\ge
    4\kappa_{\min}^2\sin^2(\theta/2).
    \label{eq:angle_error_geometry}
    \end{align}
    Since $\theta/2\in[0,\pi/2]$, this inequality yields
    \begin{equation}
        \abs{e_\alpha}
        \le
        2\arcsin\left[
        \min\left\{
        1,\frac{\norm{\vect e_c}}{2\kappa_{\min}}
        \right\}
        \right].
        \label{eq:angle_error_bound}
    \end{equation}
    The minimum with one ensures that the bound remains valid when $\norm{\vect e_c}>2\kappa_{\min}$, in which case it reduces to the trivial bound $\abs{e_\alpha}\le\pi$. The lower bound $\kappa,\kappa_d\ge\kappa_{\min}>0$ is essential for defining the bending-direction error. At the straight configuration, $\vect c=\vect0$, the bending direction is undefined. Therefore, for trajectories that cross the straight configuration, the Cartesian curvature-vector error remains well defined and should be used as the primary error measure, while the bending-direction bound in Eq.~\eqref{eq:angle_error_bound} applies only on the nonstraight domain.
\end{remark}

\subsection{Tip-position error bound}
\label{sec:tip-position-error-bound}
Using the tip-position Jacobian defined in Section~\ref{sec:tip-position-model}, define the compact and convex curvature domain $\Omega_c$ and the corresponding maximum Jacobian norm as
\begin{subequations} \label{eq:Mp_definition}
  \begin{align}
    \Omega_c&=\{\vect c\in\R^2:\norm{\vect c}\le\kappa_{\max}\},\\
    M_p&=\max_{\vect c\in\Omega_c}\norm{\mat J_p(\vect c)}_2.
  \end{align}
\end{subequations}
Since $\mat J_p$ is continuous on the compact set $\Omega_c$, the maximum is attained and $M_p<\infty$. For any $\vect c,\vect c_d\in\Omega_c$, let $\vect e_c=\vect c_d-\vect c$. By applying the fundamental theorem of calculus along the line segment connecting $\vect c$ and $\vect c_d$, one obtains
\begin{equation}
  \vect p(\vect c_d)-\vect p(\vect c)
  =
  \int_0^1
  \mat J_p \left(\vect c+s\vect e_c\right)\vect e_c\,\dd s.
  \label{eq:tip_integral}
\end{equation}
Since $\Omega_c$ is convex, the entire line segment $\vect c+s\vect e_c$, $s\in[0,1]$, remains in $\Omega_c$. Taking the Euclidean norm of \eqref{eq:tip_integral} and using the definition of $M_p$ yields
\begin{equation}
e_p=
\norm{\vect p(\vect c_d)-\vect p(\vect c)}
\le M_p\norm{\vect e_c}
<M_p\sqrt{B_x^2+B_y^2}.
\label{eq:tip_error_bound}
\end{equation}
Therefore, by defining
\begin{equation}
\varepsilon_p
=
M_p\sqrt{\varepsilon_x^2+\varepsilon_y^2},
\end{equation}
the prescribed curvature-error bounds imply
\begin{equation}
e_p(t)<\varepsilon_p,
\quad t\ge T_c,
\end{equation}
and $e_p(t)\to0$ as $\vect e_c(t)\to\vect0$. The value of $M_p$ depends on the arm length $L$, the prescribed curvature domain $\Omega_c$ (or equivalently $\kappa_{\max}$), and the selected matrix norm, and can therefore be evaluated numerically once these model parameters are specified.

\section{Numerical and reduced-order physical evaluation}
\label{sec:numerical-simulation}
This section presents four complementary layers of evidence. Section~\ref{sec:ideal_python} employs an ideal in-house Python model to isolate and verify the analytical chain. Section~\ref{sec:opencr-mujoco} evaluates the proposed method under shape-fitted feedback and examines representative physical failure modes in OpenCR--MuJoCo. Section~\ref{sec:application_example} considers a contact-free, application-inspired numerical case to demonstrate the practical relevance of the proposed framework. Finally, Section~\ref{sec:physical-experiments} reports a supervised one-dimensional reduced-order physical experiment, which serves as a proof-of-concept validation rather than a full-system experimental validation.

\subsection{In-house Python kinematic simulation}
\label{sec:ideal_python}
An explicit and reproducible fixed-step fourth-order Runge--Kutta (RK4) implementation is first used to isolate the analytical chain from dynamic and state-estimation effects. It evaluates the same cubic performance bounds $B_j$ and their derivatives $\dot B_j$, symmetric dual transformation, reference curvature velocity, six-tendon minimum-norm allocation, and continuously extended constant-curvature tip-position map as those used in Sections~\ref{sec:preliminaries}--\ref{sec:control}.

Five cases were considered, covering a fixed target, a time-varying reference, seven crossings of the straight configuration, a matched proportional baseline, and an application-inspired branch-alignment path. All five cases were completed without boundary violations. The corresponding full-window values of $\max_{t,j}|\xi_j(t)|$ were 0.6928, 0.1600, 0.1200, 0.6928, and 0.1600, respectively. The minimum-norm tendon-velocity map reconstructed the curvature with an error of $1.78\times10^{-15}~\mathrm{m}^{-1}$. For the fixed-target case, the proposed method and the matched baseline entered the prescribed terminal band at 1.704~s and 1.902~s, respectively. Figure~\ref{fig:ideal_summary} summarizes the fixed-target, time-varying, straight-crossing, and branch-alignment signals. The matched baseline is reported numerically rather than as a fifth curve panel because it serves as a controlled scalar comparison. 

\begin{figure}[htbp]
\centering
\includegraphics[width=0.48\textwidth]{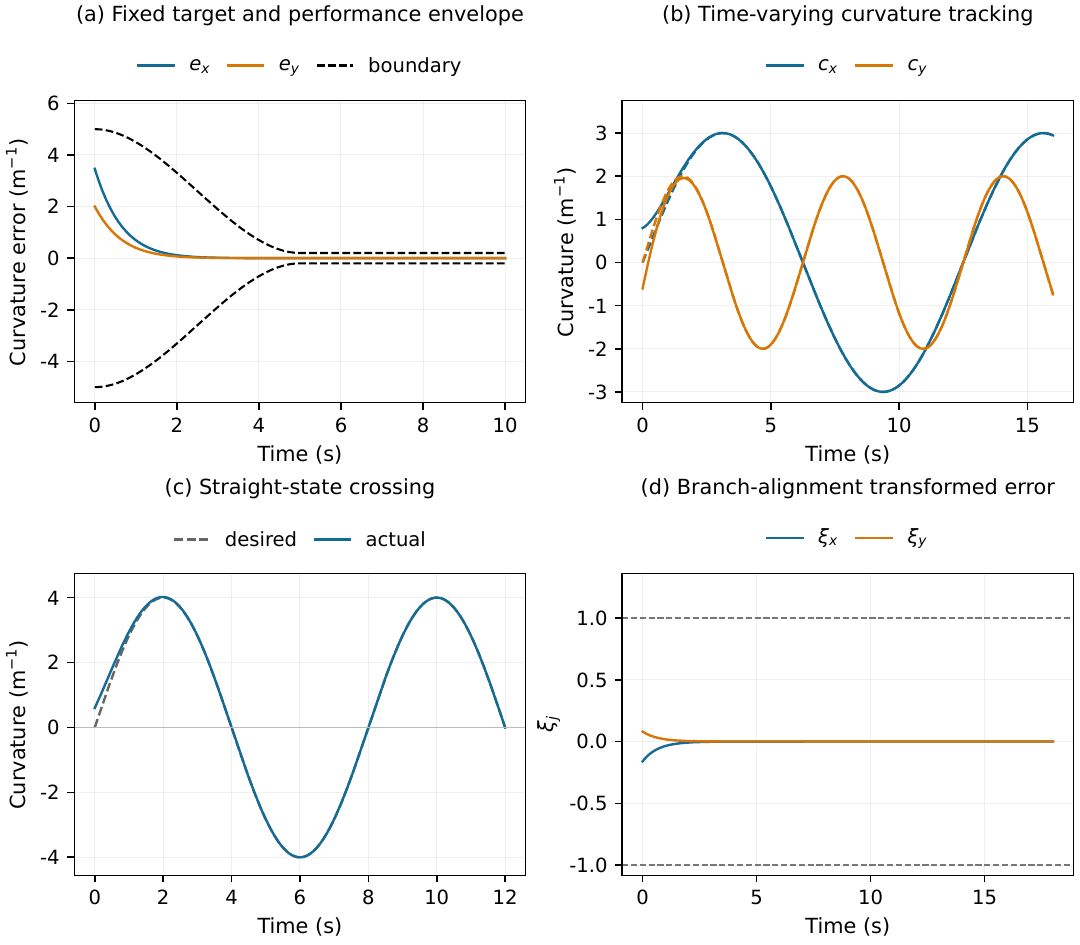}
\caption{In-house ideal kinematic layer showing fixed-target tracking, time-varying reference tracking, straight-state crossings, and branch-alignment tracking in the transformed-error space.}
\label{fig:ideal_summary}
\end{figure}

\subsection{OpenCR--MuJoCo model-based simulation}
\label{sec:opencr-mujoco}
Cases E1--E5 correspond to fixed-target tracking, time-varying reference tracking, straight-state crossing, matched-baseline comparison, and perturbation evaluation, respectively, within an internally defined 56-run benchmark suite based on the OpenCR--MuJoCo model. E6 is a separately evaluated branch-alignment run with left- and right-hold windows.

\subsubsection{Platform, estimation, and numerical settings}
\label{sec:simulation_platform}
The numerical study was conducted using the MuJoCo physics engine~\cite{todorov2012mujoco} through the external OpenCR--MuJoCo framework~\cite{shentu2026opencr}. OpenCR models the backbone as rigid links connected by elastic hinge pairs and routes six tendons through native MuJoCo actuators. The computational environment used Python 3.11.7 and MuJoCo 3.10.0. The fixed-base segment of length $L=0.300$~m was initially aligned with the local $+z$ axis. Nominal simulations were conducted under zero gravity and contact-free conditions using $N=16$ rigid links. The density, hinge stiffness, damping, and actuator parameters listed in Table~\ref{tab:simulation_parameters} are numerical settings rather than calibrated material or actuator properties. Specifically, $E$ denotes the effective numerical Young's modulus, $\rho_m$ denotes the effective numerical density used by the upstream OpenCR discretization to assign link masses and inertias, and $N\in\mathbb N$ denotes the number of rigid discretization links. Neither $E$ nor $\rho_m$ was identified from measured material properties. The discretization study in Table~\ref{tab:supp_discretization} evaluates each candidate model against the \(N=24\) reference. For \(N=16\), the differences in tip position and fitted curvature were 1.059~mm and \(0.0214~\mathrm{m}^{-1}\), respectively, while the model achieved a real-time factor of 36.3. Both discrepancies were below the predefined tolerances of 1.2~mm and \(0.1~\mathrm{m}^{-1}\), respectively. These tolerances were specified before the reported comparison as practical numerical-convergence thresholds for this study, rather than as sensor-resolution limits or industry-standard criteria. Based on these criteria, the \(N=16\) model was selected for the subsequent simulations.

\begin{table}[htbp]
    \centering
    \caption{Fixed numerical settings used in the OpenCR--MuJoCo simulations. The material-related entries are effective numerical values rather than calibrated properties.}
    \label{tab:simulation_parameters}
    \footnotesize
    \begin{tabular}{lll}
        \toprule
        Parameter & Value & Unit \\
        \midrule
        $L$ & 0.300 & m \\
        $d$ & 0.0150 & m \\
        $N$ & 16 & -- \\
        $E$ & 50 & GPa \\
        $\rho_m$ & $3.771\times10^{4}$ & kg m$^{-3}$ \\
        Hinge stiffness & 0.1309 & N m rad$^{-1}$ \\
        Hinge damping & 0.050 & N m s rad$^{-1}$ \\
        $\Delta t$ & 0.00200 & s \\
        Control period & 0.0200 & s \\
        Tendon gain & $1.00\times10^4$ & N m$^{-1}$ \\
        Pretension & 5.00 & N \\
        $\max_i |\dot q_i|$ & 0.0250 & m s$^{-1}$ \\
        $\boldsymbol B_0=[B_{x0},B_{y0}]^{\mathsf T}$ & $[5.00,5.00]^{\mathsf T}$ & m$^{-1}$ \\
        $\boldsymbol\varepsilon=[\varepsilon_x,\varepsilon_y]^{\mathsf T}$ & $[0.200,0.200]^{\mathsf T}$ & m$^{-1}$ \\
        $\boldsymbol T=[T_x,T_y]^{\mathsf T}$ & $[5.00,5.00]^{\mathsf T}$ & s \\
        $\boldsymbol K=\operatorname{diag}(k_x,k_y)$ & $\operatorname{diag}(1.50,1.50)$ & s$^{-1}$ \\
        \bottomrule
    \end{tabular}
\end{table}

\begin{table}[htbp]
\centering
\caption{Discretization study relative to \(N=24\).}
\label{tab:supp_discretization}
\footnotesize
\begin{tabularx}{\columnwidth}{lXXXX}
\toprule
N & Tip diff. (mm) & $\Delta\kappa\ (\mathrm{m}^{-1})$ & Shape RMSE (mm) & Real-time factor \\
\midrule
8 & 1.7066 & 0.0237 & 0.3468 & 82.0508 \\
12 & 1.1237 & 0.0379 & 0.3654 & 54.2026 \\
16 & 1.0588 & 0.0214 & 0.125 & 36.2731 \\
24 & 0.0 & 0.0 & 0.6263 & 19.0938 \\
\bottomrule
\end{tabularx}
\end{table}

The command convention was $q_i=l_{i,\mathrm{ref}}-l_i$, giving $l_{i,\mathrm{target}}=l_{i,\mathrm{ref}}-q_i$. In the OpenCR implementation, $l_{i,\mathrm{ref}}$ denotes the settled tendon length recorded after the
2-s pretension stage. Independent small-amplitude commands in $c_x$ and $c_y$ yielded a measured command-to-shape gain matrix of $\operatorname{diag}(1.0021,1.0016)$. Accordingly, no sign correction or post hoc coordinate rotation was applied. At each 50-Hz update, the measured backbone nodes were transformed into the fixed-base frame and fitted by nonlinear least squares to a single-segment constant-curvature centerline. The terminal tip node was excluded from the fit, and the estimate from the preceding update was used as the initial guess. The resulting shape-fitted curvature served as the controller feedback signal, while the curvature reconstructed from the measured tendon lengths was recorded only for diagnostic purposes. Across 12 static mapping cases, the shape-fit root-mean-square error (RMSE) ranged from 0.096 to 0.189~mm, while the root-mean-square (RMS) tip residual between the fitted constant-curvature model and MuJoCo ranged from 0.211 to 0.422~mm. Figure~\ref{fig:simulation_model} illustrates the modeled tendon layout, a representative single-segment constant-curvature fit, and the saved centerlines used to verify the curvature-estimation procedure.

\begin{figure}[htbp]
\centering
\includegraphics[width=0.48\textwidth]{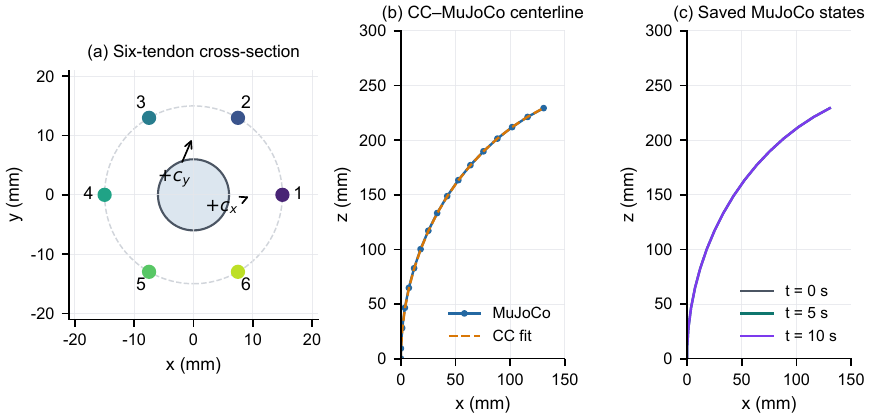}
\caption{OpenCR model and shape-estimation validation: (a) tendon numbering, (b) backbone nodes and the corresponding single-segment constant-curvature fit, and (c) three saved centerlines used for validation.}
\label{fig:simulation_model}
\end{figure}

\subsubsection{Fixed-target and time-varying tracking}
\label{sec:nominal_results}
For all OpenCR--MuJoCo tracking metrics, the curvature error and normalized error are computed from the constant-curvature shape-fitted feedback:
\begin{equation}
    \hat{\vect e}_c
    =
    \vect c_d-\hat{\vect c}_{\mathrm{shape}},
    \quad
    \hat{\vect\xi}
    =
    \mat D_B^{-1}\hat{\vect e}_c.
    \label{eq:opencr_tracking_errors}
\end{equation}
For compactness, the hats are omitted from plot labels when the feedback source is explicitly identified in the corresponding figure or text. The controller evaluated $\dot{\vect c}_d$, $B_j$, and $\dot B_j$ at a 20-ms control interval, corresponding to ten 2-ms physics steps per update. The integrated tendon-shortening commands were clipped to the specified position, rate, and actuator limits, and every sample affected by clipping was flagged. The recorded data include the desired, shape-fitted, and tendon-derived curvatures, component-wise performance bounds and shape-fitted normalized errors, requested and applied tendon commands, tendon forces, MuJoCo and constant-curvature tip positions, compatibility residuals, and saturation flags. All 56 internally defined benchmark runs were completed, including 33 perturbation runs, and none violated the prescribed performance boundaries. The separate E6 branch-alignment run is excluded from this count.

For the numerical workspace with $\kappa_{\max}=7.5~\mathrm{m}^{-1}$, numerical maximization, including the analytical limit at the straight configuration, yielded $M_p=0.0450~\mathrm{m}^2$. Define the measured
constant-curvature model residual as $\delta_{\mathrm{CC}}= \norm{\vect p(\hat{\vect c}_{\mathrm{shape}}) -\vect p_{\mathrm{MJ}}}$, where $\vect p_{\mathrm{MJ}}\in\mathbb R^3$ denotes the MuJoCo tip-position vector expressed in the fixed-base frame. Thus, $\delta_{\mathrm{CC}}$ is a scalar length measured in metres. The curvature-error contribution was quantified by $M_p\norm{\hat{\vect e}_c}$, while the physical-engine tip discrepancy was interpreted using the layered diagnostic
\begin{equation}
    \norm{\vect p(\vect c_d)-\vect p_{\mathrm{MJ}}}
    \le M_p\norm{\hat{\vect e}_c}+\delta_{\mathrm{CC}}.
\end{equation}
The first term propagates the shape-fitted curvature error through the ideal constant-curvature tip map, whereas the second term accounts for the measured residual between the single-segment constant-curvature fit to the MuJoCo backbone nodes and the MuJoCo tip position. This layered diagnostic should not be interpreted as a guaranteed position bound for physical hardware. Because $T_x=T_y=5~\mathrm{s}$, the curvature terminal time is $T_c=\max\{T_x,T_y\}=5~\mathrm{s}$. All terminal-window metrics reported below are therefore evaluated over $t\ge T_c$. Table~\ref{tab:nominal_results} reports both complete-window and terminal-window tip RMS values, thereby distinguishing the full-trajectory response from the post-entry behavior and avoiding the interpretation of the commanded E1 transition as steady-state accuracy.

\begin{table*}[htbp]
\centering
\caption{Nominal results computed over the complete and terminal sampling windows.}
\label{tab:nominal_results}
\footnotesize
\begin{tabular}{lllllll}
\toprule
Case & $\max_{t,j}|\hat{\xi}_j(t)|$ & RMS $\|\hat{\mathbf e}_c\|$ ($\mathrm{m}^{-1}$) & Full tip RMS (mm) & Tip RMS, $t\geq T_c$ (mm) & Entry (s) & Violations \\
\midrule
E1 & 0.6928 & 0.7221 & 31.34 & 0.3672 & 1.68 & 0 \\
E2 & 0.0217 & 0.0034 & 0.3433 & 0.318 & 0 & 0 \\
E3 & 0.2049 & 0.0282 & 1.305 & 1.417 & 0 & 0 \\
\bottomrule
\end{tabular}
\end{table*}

In Table~\ref{tab:nominal_results}, an entry of $0$ indicates that the initial recorded sample already satisfied the terminal-band criterion and that all subsequent samples remained within the prescribed band. E2 was preconditioned at its initial reference configuration, whereas E3 started
from the corresponding straight configuration. E1 tracked the fixed target $\vect c_d=4[\cos30^\circ,\sin30^\circ]^{\mathsf T}~\mathrm{m}^{-1}$ from the settled straight state. The maximum normalized error over the complete sampling window was $\max_{t,j}|\hat{\xi}_j(t)|=0.6928$, and both components entered and remained within their $0.2~\mathrm{m}^{-1}$ terminal bands at $1.68$~s. The complete-window tip RMS of $31.34$~mm includes the initial transition, whereas the terminal-window RMS over $t\ge T_c$ was only $0.3672$~mm, with a final tip-position error of $0.3696$~mm. The minimum tendon tension was
$2.814$~N, with neither tendon slack nor actuator saturation observed. Six additional target directions spaced by $60^\circ$ also entered the terminal bands within $1.70$--$1.76$~s without boundary violations or tendon slack. Figure~\ref{fig:fixed_target} summarizes the E1 curvature
response, component-wise performance bounds, shape-fitted normalized errors, and both complete-window and terminal-window tip metrics. 

\begin{figure*}[htbp]
\centering
\includegraphics[width=0.80\textwidth]{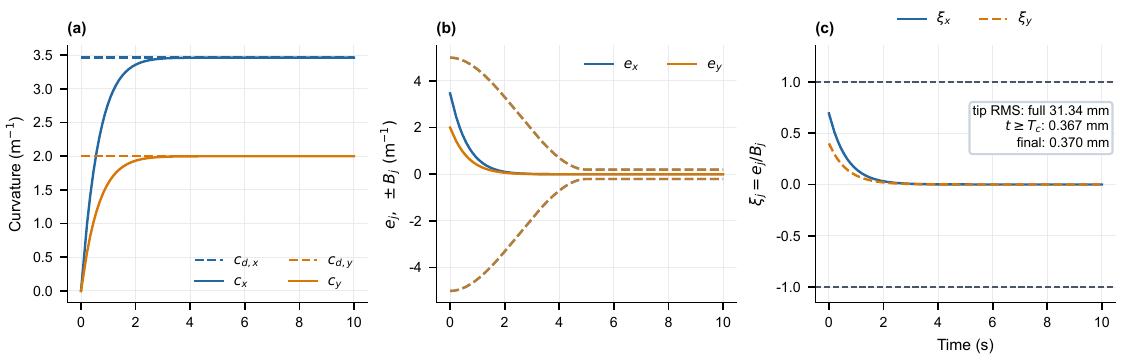}
\caption{Fixed-target E1: (a) reference and shape-fitted curvature, (b) shape-fitted errors with prescribed performance bounds, and (c) shape-fitted normalized errors with complete-window and terminal-window tip metrics.}
\label{fig:fixed_target}
\end{figure*}

E2 followed the three-period reference
\begin{equation}
    \vect c_d(t)=
    \begin{bmatrix}
        3.2+0.8\sin(2\pi\,0.125t)\\
        1.4+0.8\cos(2\pi\,0.125t)
    \end{bmatrix}
    \mathrm{m}^{-1}
\end{equation}
after preconditioning to its initial reference configuration. The maximum normalized error over the complete sampling window was $\max_{t,j}|\hat{\xi}_j(t)|=0.0217$, while the shape-fitted curvature-error RMS was $0.0034~\mathrm{m}^{-1}$. The tip RMS was $0.3433$~mm over the complete sampling window and $0.3180$~mm over the terminal window $t\ge T_c$. The peak applied tendon rate was $2.850$~mm/s. No actuator saturation, tendon slack, or prescribed-performance boundary violation was observed. Figure~\ref{fig:time_varying} shows the corresponding E2 trajectory tracking, terminal-window inset, and equal-aspect-ratio tip path.

\begin{figure*}[htbp]
\centering
\includegraphics[width=0.80\textwidth]{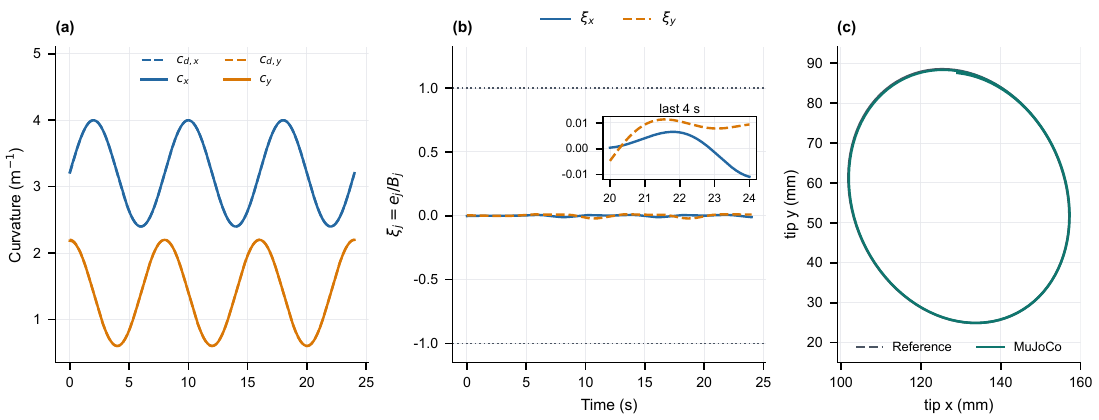}
\caption{Time-varying E2: (a) curvature tracking, (b) shape-fitted normalized errors with an inset over the final 4~s, and (c) equal-aspect-ratio desired and MuJoCo tip paths.}
\label{fig:time_varying}
\end{figure*}

\subsubsection{Straight-state crossing and controller comparison}
\label{sec:crossing_comparison}
E3 commanded $c_{d,x}=3\sin(2\pi t/4)~\mathrm{m}^{-1}$ and $c_{d,y}=0$ for $12$~s. The reference crossed the straight configuration $\vect c=\vect 0$ seven times, including the initial and
final endpoints. All signals other than the intentionally masked bending direction angle remained finite throughout the run. The direction error was omitted for $31$ of the $601$ samples because at least one curvature magnitude was below $\kappa_{\min}=0.25~\mathrm{m}^{-1}$. No angle value was interpolated across this undefined region. The maximum normalized error over the complete sampling window was $\max_{t,j}|\hat{\xi}_j(t)|=0.2049$. The complete-window tip RMS was $1.305$~mm, compared with $1.417$~mm over the terminal window $t\ge T_c$. The peak tendon rate was $21.45$~mm/s, with no actuator saturation, tendon slack, or prescribed-performance boundary violation observed. These results demonstrate the numerical regularity of the Cartesian-vector implementation at $\vect c=\vect0$, rather than the existence of a well-defined bending direction at the straight configuration. Figure~\ref{fig:straight_crossing} reports the E3 signed curvature, phase-aligned straight-state crossings, tendon shortening, and tip-position norm without plotting the undefined bending direction angle.

\begin{figure*}[htbp]
\centering
\includegraphics[width=0.80\textwidth]{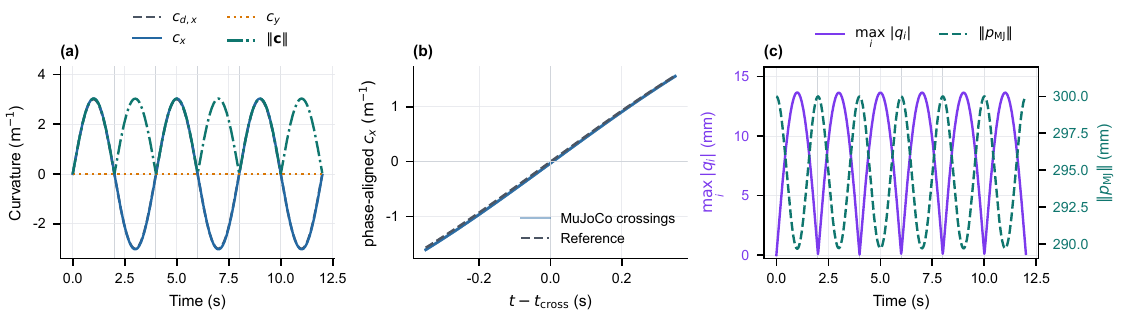}
\caption{Straight-state E3: (a) signed curvature and curvature magnitude, (b) phase-aligned zero crossings, and (c) tendon-shortening magnitude and MuJoCo tip-position norm. No bending-direction angle is plotted.}
\label{fig:straight_crossing}
\end{figure*}

The E4 baseline used $\dot{\vect c}_r=\dot{\vect c}_d+\mat K\hat{\vect e}_c$, while retaining
the same initial state, feedback estimator, control gain, model, sampling period, and command limits as the proposed controller. Both methods remained within the prescribed performance envelopes throughout the run. Compared with the baseline, the proposed method reduced the shape-fitted curvature-error RMS and full-window tip RMS by $1.99\%$ and shortened the terminal-band entry time from $1.86$ to $1.68$~s, corresponding to a $9.68\%$ reduction, at the cost of a $4.41\%$ increase in the integrated squared tendon-rate effort. The comparison therefore indicates a modest improvement in transient tracking performance rather than a stability
advantage arising from an unstable baseline. Table~\ref{tab:controller_comparison} provides the complete comparison metrics.

\begin{table*}[htbp]
\centering
\caption{Matched controller comparison.}
\label{tab:controller_comparison}
\footnotesize
\begin{tabular}{llllll}
\toprule
Controller & RMS $\|\hat{\mathbf e}_c\|$ ($\mathrm{m}^{-1}$) & Entry time (s) & Full tip RMS (mm) & $\int\|\dot{\mathbf q}\|^2dt$ ($\mathrm{m}^2\mathrm{s}^{-1}$) & $\max_{t,j}|\hat{\xi}_j(t)|$ \\
\midrule
Proposed & 0.7221 & 1.68 & 31.34 & 0.000776 & 0.6928 \\
Baseline & 0.7368 & 1.86 & 31.98 & 0.000743 & 0.6928 \\
\bottomrule
\end{tabular}
\end{table*}

\subsubsection{Perturbation and failure-mode evidence}
\label{sec:robustness_results}
E5 comprised 13 one-at-a-time (OAT) variations and 20 fixed-seed joint-random samples. Variations covered stiffness \(\pm20\%\), damping \(\pm30\%\), tendon gain and pretension \(\pm20\%\), controller radius mismatch \(\pm10\%\), gravity, measurement noise, and a two-cycle delay. Figure~\ref{fig:robustness} reports terminal-window (\(t\ge T_c\)) metrics rather than the uninformative initial maximum shared by all fixed-target runs. Across all 33 perturbation
runs, \(\max_{t\ge T_c,j}\abs{\hat\xi_j(t)}\) did not exceed 0.2135 and no performance boundary was violated. Gravity produced a terminal tip RMS of 28.03 mm and constant-curvature model residual of 28.46 mm, illustrating why the ideal curvature bound cannot directly constrain the MuJoCo tip without a residual term. Three joint-random samples reached zero minimum tension, with logged slack durations of 10.00--10.02\~s. These cases are retained in both the figure and the underlying data. Some samples also showed small saturation fractions up to 0.399\%, yet none crossed a prescribed error boundary. These empirical outcomes do not extend the nominal proof to unilateral tendon feasibility. The complete 13-row one-at-a-time table is provided as Table~\ref{tab:supp_robustness}.

\begin{figure*}[htbp]
\centering
\includegraphics[width=0.70\textwidth]{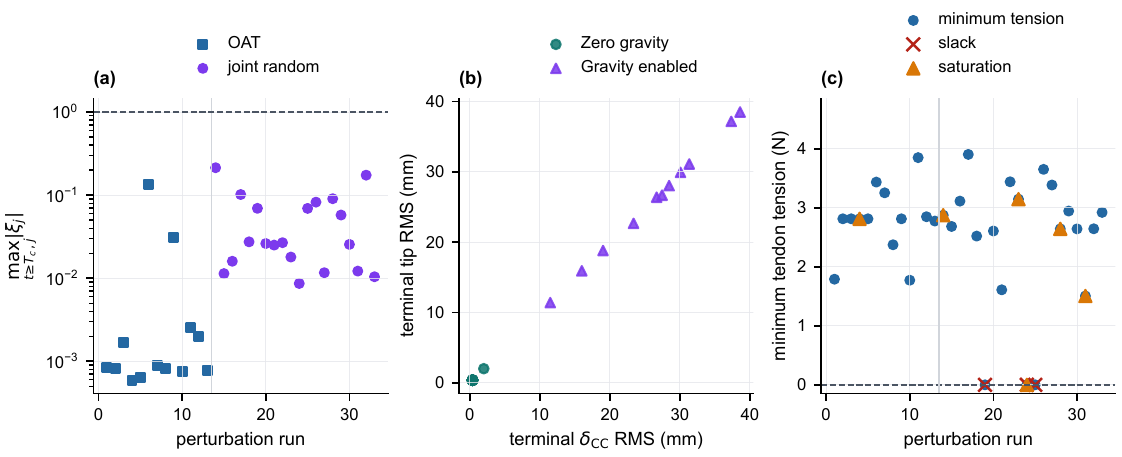}
\caption{E5 perturbation evidence over $t\ge T_c$: (a) terminal shape-fitted normalized-error maxima, (b) terminal tip error versus constant-curvature model residual, and (c) tension, slack, and saturation diagnostics. All failed-tension samples are retained.}
\label{fig:robustness}
\end{figure*}

\begin{table*}[htbp]
\centering
\caption{Complete one-at-a-time robustness results for \(t\ge T_c\).}
\label{tab:supp_robustness}
\footnotesize
\begin{tabularx}{\textwidth}{lXXXXXXX}
\toprule
Condition & $\max_{t\geq T_c,j}|\xi_j|$ & Tip RMS, $t\geq T_c$ (mm) & $\delta_{\mathrm{CC}}$ RMS, $t\geq T_c$ (mm) & Min. tension (N) & Slack (s) & Saturation & Violations \\
\midrule
Damping -30\% & 0.0008 & 0.3682 & 0.3696 & 1.7903 & 0.0 & 0.0 & 0 \\
Damping +30\% & 0.0008 & 0.3667 & 0.3688 & 2.8136 & 0.0 & 0.0 & 0 \\
Radius model -10\% & 0.0017 & 0.3659 & 0.369 & 2.8136 & 0.0 & 0.0 & 0 \\
Radius model +10\% & 0.0006 & 0.3679 & 0.3692 & 2.8136 & 0.0 & 0.0027 & 0 \\
Two-cycle delay & 0.0006 & 0.3677 & 0.3692 & 2.8136 & 0.0 & 0.0 & 0 \\
Gravity & 0.1358 & 28.0331 & 28.4599 & 3.4356 & 0.0 & 0.0 & 0 \\
Tendon gain -20\% & 0.0009 & 0.3663 & 0.3687 & 3.2536 & 0.0 & 0.0 & 0 \\
Tendon gain +20\% & 0.0008 & 0.368 & 0.3694 & 2.3736 & 0.0 & 0.0 & 0 \\
Measurement noise & 0.0312 & 0.369 & 0.3868 & 2.8132 & 0.0 & 0.0 & 0 \\
Pretension -20\% & 0.0008 & 0.3685 & 0.3696 & 1.7728 & 0.0 & 0.0 & 0 \\
Pretension +20\% & 0.0025 & 0.3942 & 0.3825 & 3.8503 & 0.0 & 0.0 & 0 \\
Young's modulus -20\% & 0.002 & 0.3875 & 0.3786 & 2.8486 & 0.0 & 0.0 & 0 \\
Young's modulus +20\% & 0.0008 & 0.3684 & 0.3696 & 2.7755 & 0.0 & 0.0 & 0 \\
\bottomrule
\end{tabularx}
\end{table*}

\subsection{Application-inspired branch alignment in a bifurcated conduit}
\label{sec:application_example}
E6 represents a generic contact-free branch-alignment problem in a schematic bifurcated conduit. This task may arise in industrial pipe or duct inspection, flexible-probe orientation, and flexible bronchoscope branch steering~\cite{song2026bronchoscope}. It commands a straight approach followed by left- and right-branch curvature holds. The flexible segment provides curvature steering only, while axial insertion would require an external linear module and is not modeled in this study. Over the complete E6 run, the componentwise maxima were
\(\max_t\abs{\hat\xi_x(t)}=0.0841\) and
\(\max_t\abs{\hat\xi_y(t)}=0.0695\), with zero violations, saturation, and
slack. These values are the \(x\)- and \(y\)-component maxima, not metrics for
the left- and right-hold windows. Shape-fit RMSE was 0.0984 mm, MuJoCo tip RMS was 0.514 mm, peak tendon rate was 9.832 mm/s, and minimum tension was 2.803 N. During the two hold windows,
curvature/tip/direction RMS were
\(8.87\times10^{-4}~\mathrm{m}^{-1}\)/0.380 mm/\(0.0142^\circ\) and
\(2.80\times10^{-3}~\mathrm{m}^{-1}\)/0.463 mm/\(0.0428^\circ\).
Figure~\ref{fig:application_branch}(a) overlays the archived left/right desired centerlines to make the branch-selection intent explicit; panels (b) and (c) compare each target constant-curvature centerline with the single-segment constant-curvature fit to MuJoCo backbone nodes and mark the measured MuJoCo tip. That shape fit closes the controller loop, whereas tendon-derived curvature remains diagnostic only.

\begin{figure*}[htbp]
\centering
\includegraphics[width=0.70\textwidth]{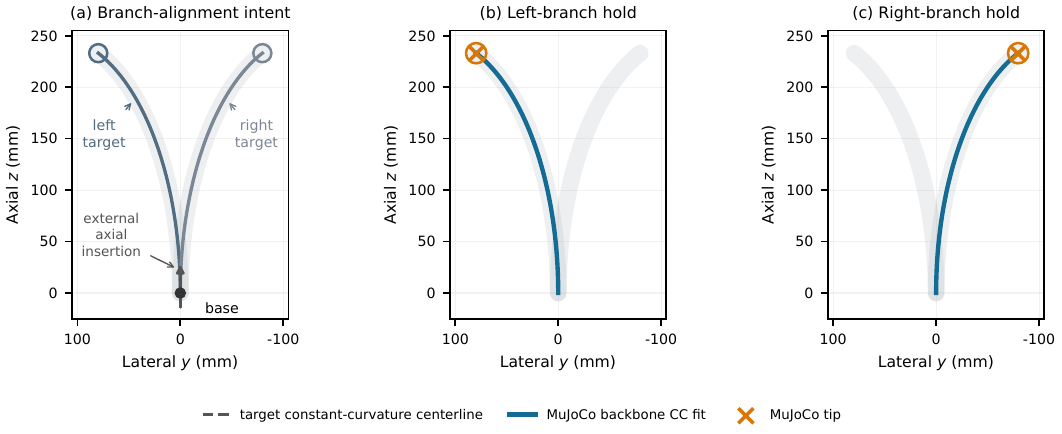}
\caption{Contact-free OpenCR--MuJoCo E6 in a schematic bifurcated conduit:
(a) left/right target-branch alignment intent formed from the archived desired
centerlines, (b) left-branch hold, and (c) right-branch hold. The hold panels
compare the target constant-curvature centerline with the shape fit to MuJoCo backbone nodes
and mark the MuJoCo tip. The pale Y-shaped guide is schematic, and the base arrow
denotes axial insertion that is not modeled.}
\label{fig:application_branch}
\end{figure*}

\subsection{Reduced-order physical experiment}
\label{sec:physical-experiments}
\paragraph{Platform and measurement} The available physical platform is a two-section, four-channel tendon-driven robot, whereas Sections~\ref{sec:preliminaries}--\ref{sec:control} consider a single-section plant with six tendons and a two-dimensional curvature vector. The physical experiment therefore adopted a reduced antagonistic actuation scheme,
\(\left[M_1,M_2,M_3,M_4\right]\trans=[1,-1,1,-1]\trans u\), selected the upper section, and regulated a single baseline-referenced signed planar-curvature component. The six-tendon minimum-norm allocation was not implemented on the physical platform. Here, \(u\) denotes a motor-command coordinate rather than the bending angle of the robot. Six formal runs were conducted using 12-s motor target-command trajectories updated at 2~Hz, resulting in 24 command updates per run. Continuous encoder-based motor-angle trajectories were not recorded. Instead, one GT3 single-turn encoder feedback snapshot was retained before motion and one after the robot returned to its initial configuration for each formal run.

A fixed iPhone camera and nine physical markers provided an independent measurement of the robot shape. Markers \(\texttt{R0}\) and \(\texttt{R1}\) defined the image translation, rotation, and scale references. Markers \(\texttt{B0}\), \(\texttt{U1}\), \(\texttt{U2}\), and \(\texttt{J}\) defined the selected upper section, while \(\texttt{J}\), \(\texttt{L1}\), \(\texttt{L2}\), and \(\texttt{T}\) provided a separate diagnostic measurement of the lower section that was not used for control. The marker coordinates were exponentially smoothed with a coefficient of 0.42. The selected upper-section markers were fitted to a single-segment constant-curvature arc using robust nonlinear least squares. The curvature feedback was then obtained as the median of the valid upper-section curvature estimates over the latest 0.2~s. Although the physical platform contains two mechanical sections, the reported control experiment considered only one section and reconstructed a single signed planar-curvature component using a single-segment constant-curvature approximation. No piecewise-constant-curvature control model was introduced. IMU pitch was used only as auxiliary posture information. The control and computation, actuation, and sensing and feedback modules, together with their implemented information flows, are shown in Figure~\ref{fig:hardware_platform_pass7}.

\begin{figure}[htbp]
\centering
\includegraphics[width=0.50\textwidth]{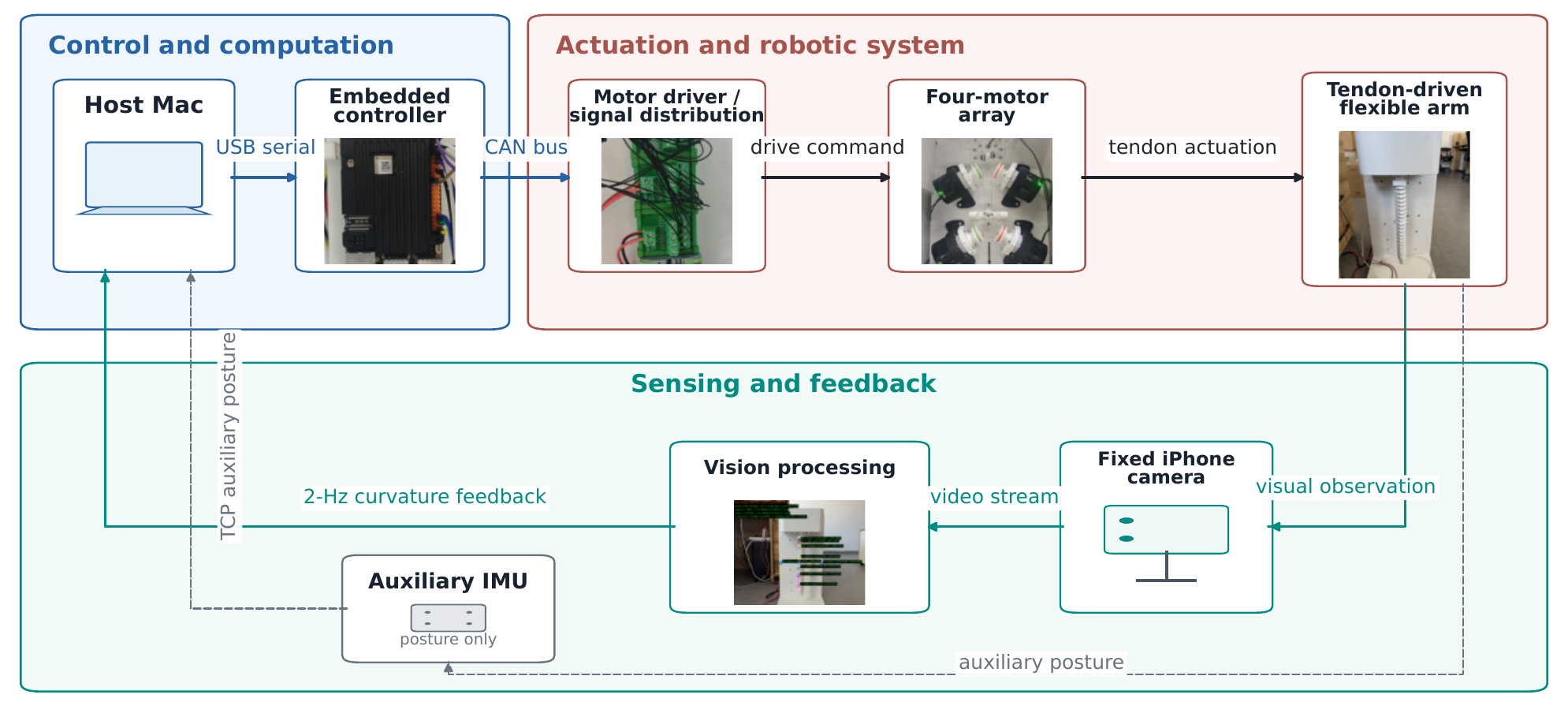}
\caption{Architecture and information flow of the reduced-order physical
platform. The host computer sends motor-target commands through the embedded
controller and the motor-drive chain to the four-motor tendon actuation system.
A fixed iPhone provides the primary visual stream used for curvature
reconstruction, whereas the IMU is retained only as an auxiliary posture signal.}
\label{fig:hardware_platform_pass7}
\end{figure}

\paragraph{Calibration and implementation} Six open-loop calibration trials were conducted at motor-command levels of 10 and 15 degrees, with three repetitions at each level. The upper section was selected for the subsequent control experiment because its higher response signal-to-noise ratio, better repeatability, and substantially lower inter-section coupling outweighed its somewhat larger arc-fitting residual. Over the two directly tested command levels, the local calibration was described by
\begin{align}
\kappa_{\mathrm{rel}}&=g u+b, \notag \\
g&=-0.05575~\mathrm{m}^{-1}/\text{motor-command degree}, \label{eq:hardware_local_calibration} \\
b&=-0.02460~\mathrm{m}^{-1},\quad R^2=0.9979. \notag
\end{align}
Because \(\kappa_{\mathrm{rel}}\) is defined relative to the pre-command baseline of each trial, \(u=0\) does not constitute an independent third calibration level. The formal target \(\kappa_d=-0.465~\mathrm{m}^{-1}\) corresponds to \(u_{\mathrm{eq}}=(\kappa_d-b)/g=7.90\) motor-command degrees, which requires a modest extrapolation below the directly tested range of 10--15 degrees. The resulting fit is therefore used only to determine the local input sign and gain around the tested operating range and is not intended as a global model of the robot. Five frames from the loading interval of the accepted calibration trial \texttt{CAL15\_R2} are shown in Figure~\ref{fig:hardware_motion_frames}.

\begin{figure*}[htbp]
\centering
\begin{minipage}[t]{0.15\textwidth}
\centering
\includegraphics[width=\linewidth]{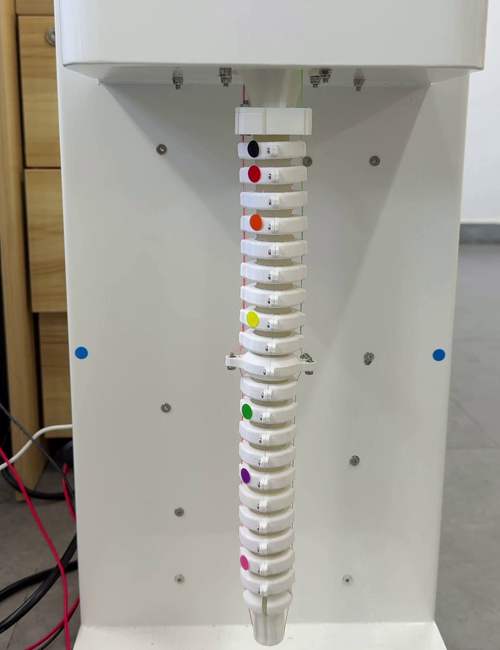}\\[-0.5ex]
\small (a) baseline
\end{minipage}\hfill
\begin{minipage}[t]{0.15\textwidth}
\centering
\includegraphics[width=\linewidth]{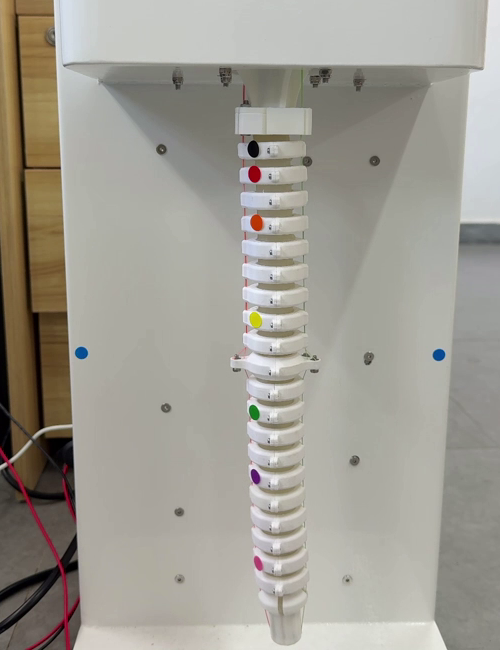}\\[-0.5ex]
\small (b) 0.7 s
\end{minipage}\hfill
\begin{minipage}[t]{0.15\textwidth}
\centering
\includegraphics[width=\linewidth]{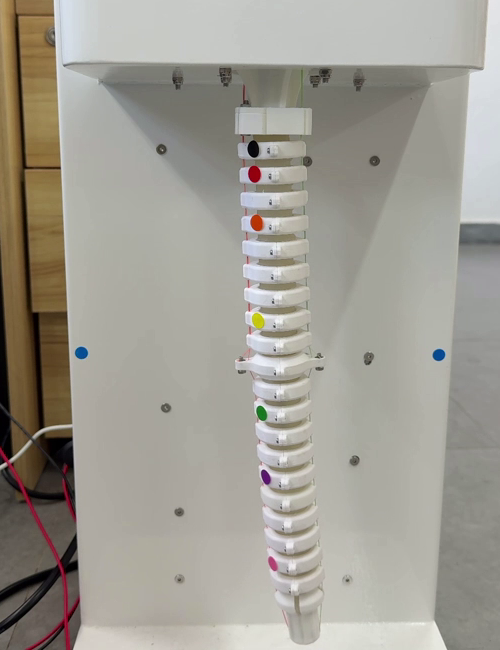}\\[-0.5ex]
\small (c) 1.2 s
\end{minipage}
\hfill
\begin{minipage}[t]{0.15\textwidth}
\centering
\includegraphics[width=\linewidth]{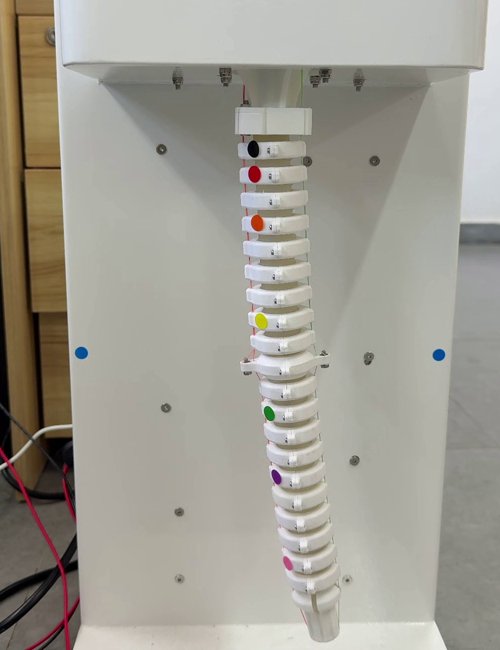}\\[-0.5ex]
\small (d) 1.8 s
\end{minipage}
\hfill
\begin{minipage}[t]{0.15\textwidth}
\centering
\includegraphics[width=\linewidth]{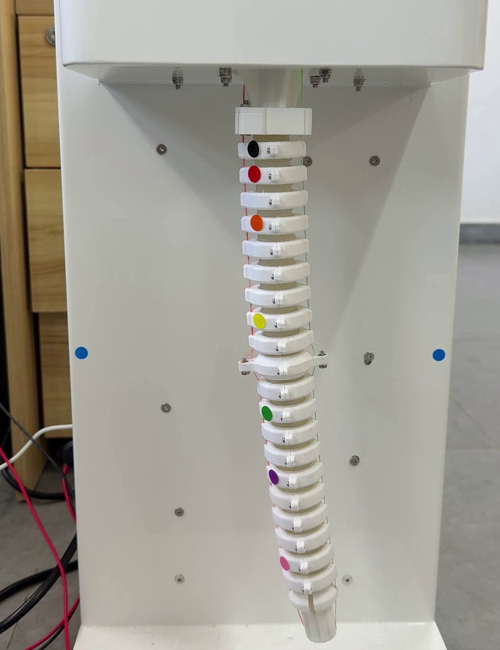}\\[-0.5ex]
\small (e) 2.4 s
\end{minipage}
\caption{Unannotated camera frames from the loading interval of accepted
15-motor-command-degree calibration trial \texttt{CAL15\_R2}: (a) the
pre-command baseline at \(u=0\); (b)--(e) intermediate rightward bending at
\(t=0.7\), \(1.2\), \(1.8\), and \(2.4~\mathrm{s}\) after command onset.}
\label{fig:hardware_motion_frames}
\end{figure*}

Trial P2 was selected to illustrate the measurement chain because it retained all 24 scheduled control updates and 360 valid upper-section vision samples over the complete 0--12~s tracking interval. The corresponding saved visual sampling rate was 29.99~Hz. Figure~\ref{fig:hardware_measurements_teacher}(b) subtracts the initial coordinate of each marker only and does not rescale the recorded displacements.

\begin{figure}[htbp]
\centering
\includegraphics[width=0.50\textwidth]{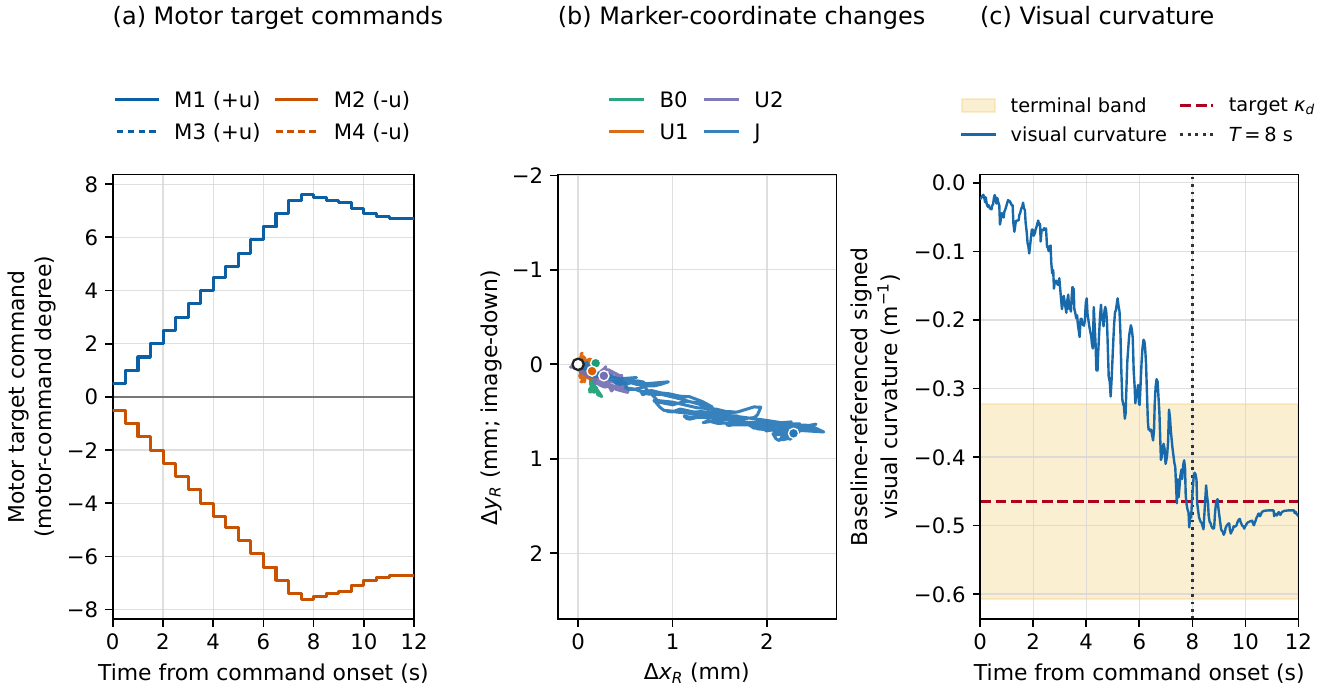}
\caption{Recorded measurement chain for accepted formal trial P2: (a) four
motor target-command trajectories, with \(M_1=M_3=+u\) and \(M_2=M_4=-u\);
(b) unscaled \(\texttt{R0}/\texttt{R1}\)-reference-frame changes of the
\(\texttt{B0}\), \(\texttt{U1}\), \(\texttt{U2}\), and \(\texttt{J}\)
marker coordinates relative to \(t=0\) (open common origin and filled final
points); and (c) the corresponding baseline-referenced signed visual
curvature, target, terminal band, and prescribed time \(T=8~\mathrm{s}\).
Panel (a) shows commanded motor targets, not continuous encoder-measured motor
angles.}
\label{fig:hardware_measurements_teacher}
\end{figure}

At sampled time \(t_k\), the filtered visual curvature gives
\(e_{h,k}=\kappa_d-\hat\kappa_k\), while
\(B_{h,k}=B_h(t_k)\) and \(\dot B_{h,k}=\dot B_h(t_k)\) evaluate the same
analytic cubic envelope for both controllers. The hardware scalar normalized
error is \(\xi_{h,k}=e_{h,k}/B_{h,k}\), abbreviated as \(\xi\) in
Figure~\ref{fig:hardware_results_pass7} and Table~\ref{tab:hardware_results_pass8}.
The proposed update uses
\(\dot\kappa_{r,k}=\dot\kappa_{d,k}-(\dot B_{h,k}/B_{h,k})e_{h,k}+k e_{h,k}\),
whereas the matched baseline uses
\(\dot\kappa_{r,k}=\dot\kappa_{d,k}+k e_{h,k}\). With
\(\Delta t=0.5~\mathrm{s}\), the recorded code implements
\begin{align}
  \dot u_k^\star&=\dot\kappa_{r,k}/g, &
  \dot u_k&=\operatorname{clip}_{[-1,1]}(\dot u_k^\star),\nonumber\\
  \Delta u_k&=\operatorname{clip}_{[-0.5,0.5]}(\Delta t\,\dot u_k), &
  \tilde u_{k+1}&=\operatorname{clip}_{[0,12]}(u_k+\Delta u_k),\nonumber\\
  u_{k+1}&=\mathcal Q_{\mathrm{proto}}(\tilde u_{k+1};u_k), &
  \mathbf M_{k+1}&=[1,-1,1,-1]^{\mathrm T}u_{k+1}.
  \label{eq:hardware_discrete_update}
\end{align}
The rate, step, and position clip intervals have units of motor-command
degree/s, motor-command degree, and motor-command degree, respectively.
The local intercept \(b\) is used in calibration and the static-equivalent
command calculation, but the derivative inversion in
Eq.~\eqref{eq:hardware_discrete_update} uses \(g\). The value
\(\mathcal Q_{\mathrm{proto}}\) selects the nearest reachable point on the
0.1-motor-command-degree grid supported by the legacy serial-command protocol,
subject to the 0--12 motor-command-degree position limit and the 0.5-motor-command-degree
step limit from \(u_k\). It rejects payloads containing a premature legacy
frame terminator and chooses the lower magnitude on an exact tie. The applied
command is held until the next 2-Hz outer-loop update.

\begin{figure}[htbp]
\centering
\includegraphics[width=0.48\textwidth]{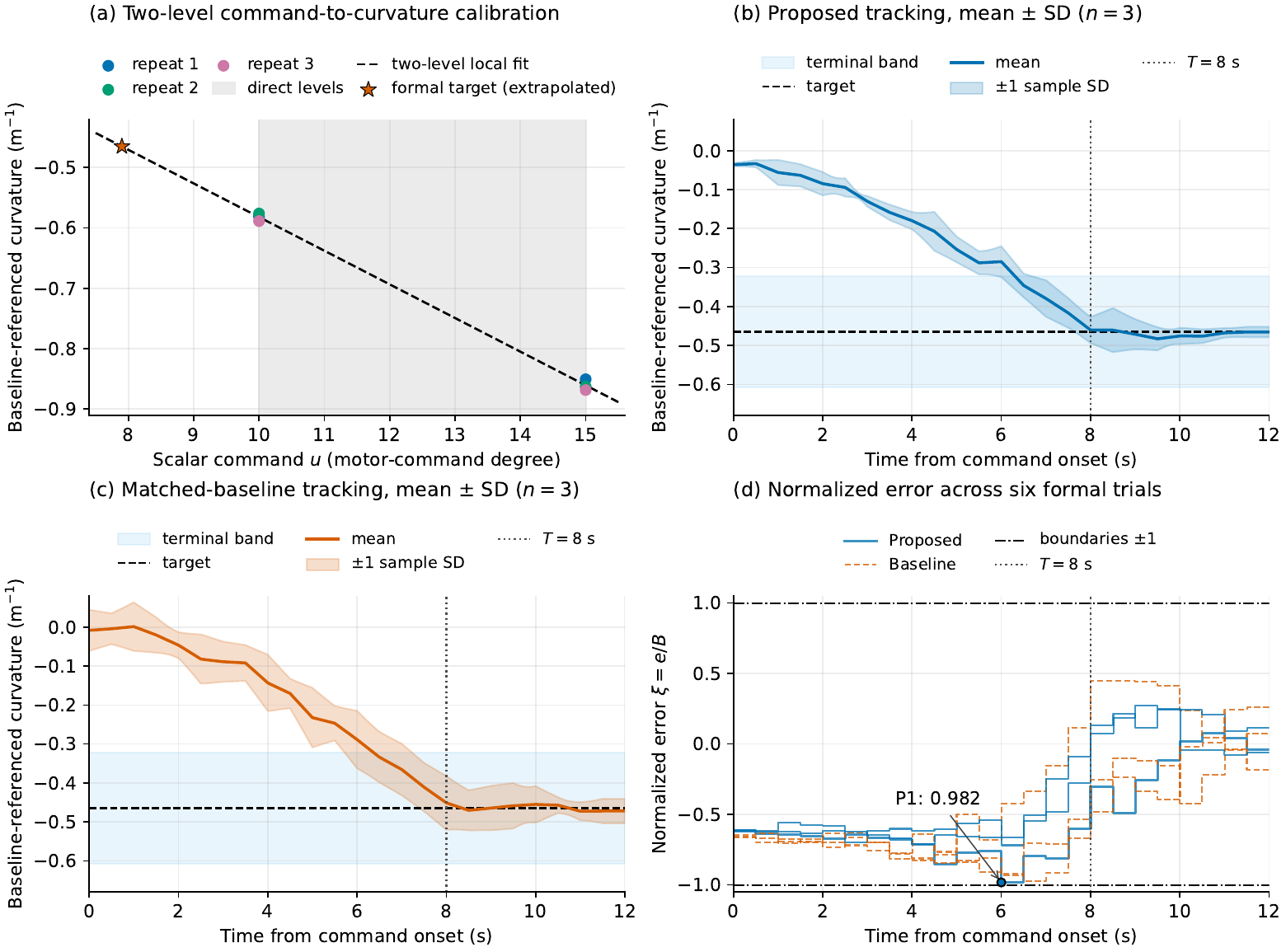}
\caption{Reduced-order hardware results: (a) two-level local
command-to-curvature calibration based only on the directly tested 10/15
levels; (b) proposed and (c) matched-baseline tracking shown as mean \(\pm\)
sample SD (\(n=3\)); and (d) normalized errors for all six formal trials with
the \(\pm1\) boundaries. All six trials entered the terminal band before
\(T=8~\mathrm{s}\).}
\label{fig:hardware_results_pass7}
\end{figure}

Each run used a 12-s post-onset tracking interval with 24 updates and froze
\begin{equation}
    B_{h,0}=1.25\abs{e_{h,0}}+4\sigma_{\mathrm{base}},\quad
  \varepsilon_h=\max\{4\sigma_{\mathrm{base}},0.10\abs{\kappa_d}\}.
\end{equation}
The baseline noise level is defined as
$
\sigma_{\mathrm{base}}
=
\max\left\{
\sigma_{H0},
s_1\left(\{\widetilde\kappa_{0.2}(t_i)\}_{i=1}^{N}\right)
\right\}.
$
Here, $\widetilde\kappa_{0.2}$ denotes the median valid visual curvature of the selected section over the latest 0.2~s after exponential smoothing of the marker coordinates with a coefficient of 0.42. The quantity $s_1$ denotes the sample standard deviation ($\mathrm{ddof}=1$) of these values sampled every 0.05~s during the separate 3-s pre-motion baseline of each run.
The fixed H0 floor is
\(\sigma_{H0}=0.0355617~\mathrm{m}^{-1}\), defined in the metric code as the
maximum upper-section curvature sample SD (\(\mathrm{ddof}=1\)) across three
separate approximately 60-s read-only static H0 runs. It was frozen in the
calibration file before formal testing and shared by all six accepted formal
runs. Each run independently computed and then froze its $\sigma_{\mathrm{base}}$ and $B_{h,0}$ according to the maximum rule above. In all six accepted runs, the H0 floor was the active bound, resulting in a recorded $\sigma_{\mathrm{base}}$ of $0.0355617~\mathrm{m}^{-1}$.
The identical
rule was used for the proposed controller and matched baseline. The campaign used $\varepsilon_h=0.142~\mathrm{m}^{-1}$, $T=8~\mathrm{s}$, and $k=0.5~\mathrm{s}^{-1}$. The target, filtering procedure, command limits, and return routine were otherwise identical across all runs.
The firmware speed field was 100, but
its physical unit was not independently calibrated.

\paragraph{Closed-loop results and scope} All six formal trials entered the terminal band before $T=8~\mathrm{s}$. The latest entry occurred at $7.502~\mathrm{s}$. Every
measured normalized-error trajectory remained inside the boundary, with the
largest run-level value \(\max\abs{\xi_h}=0.9821\) and zero observed violations
across the six trials. The proposed and matched-baseline terminal curvature
RMSE values were \(0.0267\pm0.0054\) and
\(0.0396\pm0.0061~\mathrm{m}^{-1}\), respectively, giving a descriptive
reduction of 32.5\%. Their mean terminal-band entry times were nearly identical,
and integrated squared command rates were also similar. Values in Table~\ref{tab:hardware_results_pass8} are reported as the mean $\pm$ sample SD across three runs for each controller. No statistical significance is claimed.

\begin{table}[htbp]
\centering
\caption{Descriptive reduced-order hardware comparison.}
\label{tab:hardware_results_pass8}
\footnotesize
\setlength{\tabcolsep}{2.2pt}
\begin{tabularx}{\columnwidth}{@{}>{\raggedright\arraybackslash}Xcc@{}}
\toprule
Metric (unit) & Proposed & Matched baseline \\
\midrule
Terminal-band entry (s) & 6.668 $\pm$ 0.288 & 6.670 $\pm$ 0.763 \\
Terminal curvature RMSE (m$^{-1}$) & 0.0267 $\pm$ 0.0054 & 0.0396 $\pm$ 0.0061 \\
Maximum $|\xi|$ (--) & 0.804 $\pm$ 0.155 & 0.887 $\pm$ 0.103 \\
Final signed bias (m$^{-1}$) & 0.0016 $\pm$ 0.0123 & 0.0023 $\pm$ 0.0191 \\
Return residual (m$^{-1}$) & 0.0660 $\pm$ 0.0230 & 0.0676 $\pm$ 0.0208 \\
Integrated squared command rate ((motor-command degree)$^2$/s) & 7.178 $\pm$ 0.488 & 7.246 $\pm$ 0.547 \\
\bottomrule
\end{tabularx}
\parbox{0.98\columnwidth}{\vspace{2pt}\scriptsize Values are mean $\pm$ sample SD ($\mathrm{ddof}=1$; $n=3$/controller), with no significance test. Entry is the first 2-Hz sample after which $|e_h|\leq\varepsilon_h$ holds for all remaining tracking samples. Terminal RMSE is computed over $t\in[8,12]$~s, while bias is defined as the mean $e_h$ over the final 1.0-s control window. Return residual is the absolute difference between the mean upper-section returned curvature over the final 1.0~s and the frozen calibration-baseline curvature, defined as the median of the six calibration-run pre-command baseline means. Effort is $\sum_k(\Delta u_k/\Delta t_k)^2\Delta t_k$ for applied $u_k$ over the full tracking record.}
\end{table}

Accepted runs spanned 15.69--29.68 processed frames/s, all above the fixed
15-frames/s validity threshold, while the asynchronous 2-Hz outer loop retained
all 24 scheduled updates. Technically invalid or safety-aborted attempts were
retained in the archive and excluded according to the documented validity
criteria. One
retained run produced a runtime return warning, while the offline final-window
criterion remained satisfied. The
evidence supports supervised one-dimensional implementability and measured
envelope monitoring only, not full six-tendon allocation, two-dimensional
curvature control, tension or slack behavior, robust safety, or deployment.

\subsection{Limitations and practical implications}
The two numerical layers answer different questions. The in-house model verifies the
implemented algebra, whereas OpenCR--MuJoCo reveals estimator residuals,
gravity sensitivity, saturation, and unilateral-tendon failure modes. Across the internally defined 56-run benchmark suite, the mean constant-curvature fit RMSE was 0.171 mm under zero gravity and 8.646 mm with gravity enabled. The corresponding constant-curvature--MuJoCo tip residuals were 0.393 mm and 23.251 mm, respectively.
Twelve runs contained saturation and
three contained slack. These results support the narrow nominal theorem and
simultaneously show why it cannot be read as a robust hardware guarantee. Practical deployment therefore requires workspace-specific constant-curvature model validation,
gravity/load compensation, calibrated rate and displacement margins,
positive-tension monitoring, and independent backbone shape sensing.
Friction, hysteresis, contact, torsion, extension, and parameter drift must
be incorporated into a revised model or bounded-disturbance analysis. The physical layer narrows but does not close these gaps. It coordinates four
channels through one scalar input rather than six tendons in two-dimensional
curvature space, tests one target direction, uses a two-level local calibration
with modest extrapolation below the direct levels, and contains only three runs
per controller. The fixed-camera estimator exhibited run-to-run variation in processing rate. The platform provided no tension sensing or multi-turn absolute-position measurement, and the physical unit of the firmware speed field was not independently calibrated. Consequently, the descriptive RMSE difference cannot be
generalized to the full system or interpreted as a robust hardware guarantee.

\FloatBarrier

\section{Conclusion}
\label{sec:conclusions}
This study investigated curvature tracking for a single-segment, six-tendon flexible arm.
The Cartesian curvature vector permits continuous passage
through the straight state, and the explicit tendon relations support a
compatible minimum-norm velocity allocation. The control scheme combines static symmetric error scaling with a cubic time-varying performance-bound design that progressively shrinks the admissible error envelope.
Under the stated ideal conditions and initial
feasibility, the closed loop has boundary invariance, enters the selected
nonzero terminal band by the prescribed time, and converges asymptotically. The in-house Python cases and the internally defined 56-run suite evaluated
using the OpenCR--MuJoCo model provide two complementary layers of numerical
evidence. The separate contact-free example further illustrates the relevance
of the control formulation to generic confined-space branch-alignment tasks.
Gravity enlarged constant-curvature model residuals, twelve runs showed
saturation, and three
showed slack, thereby exposing the boundary of the nominal theorem. The supervised reduced-order experiment provides a distinct hardware
feasibility layer. Across six formal trials, all measured normalized-error
trajectories remained within the prescribed boundary and all trials entered the
terminal band before \(T=8~\mathrm{s}\). The proposed controller showed a
descriptively 32.5\% lower mean terminal curvature RMSE than the matched
baseline, while their mean entry times were nearly identical. With three
repetitions per controller, one-dimensional coordination, local calibration,
and no tension sensing, the result is feasibility evidence rather than
full-system validation. Future work will extend the hardware validation to a full six-tendon platform and two-dimensional curvature tracking, while incorporating gravity and load compensation, direct tension and slack sensing, and contact-aware control. Larger experimental sample sizes will also be considered to provide more comprehensive evidence of repeatability and practical performance.

\appendix
\section{Full six-tendon map and exact inverse}
\label{sec:supp_mapping}
From \eqref{eq:single_tendon_vector}, we obtain
\begin{align}
  q_1&=dLc_x, \nonumber\\
  q_2&=dL\left(\tfrac12c_x+\tfrac{\sqrt3}{2}c_y\right), \nonumber\\
  q_3&=dL\left(-\tfrac12c_x+\tfrac{\sqrt3}{2}c_y\right), \nonumber\\
  q_4&=-dLc_x, \nonumber\\
  q_5&=dL\left(-\tfrac12c_x-\tfrac{\sqrt3}{2}c_y\right), \nonumber\\
  q_6&=dL\left(\tfrac12c_x-\tfrac{\sqrt3}{2}c_y\right).
  \label{eq:supp_expanded_q}
\end{align}
Consequently,
\begin{equation}
  q_4=-q_1,\quad q_5=-q_2,\quad q_6=-q_3,
  \quad q_1-q_2+q_3=0 .
  \label{eq:supp_compatibility}
\end{equation}

The Gram matrix calculation is
\begin{align}
  \mat C\trans\mat C
  &=
  \begin{bmatrix}
    \sum_{i=1}^{6}\cos^2\phi_i&
    \sum_{i=1}^{6}\cos\phi_i\sin\phi_i\\
    \sum_{i=1}^{6}\cos\phi_i\sin\phi_i&
    \sum_{i=1}^{6}\sin^2\phi_i
  \end{bmatrix}\nonumber\\
  &=
  \begin{bmatrix}3&0\\0&3\end{bmatrix}
  =3\mat I_2.
  \label{eq:supp_CtC}
\end{align}
Premultiplication of Eq.~\eqref{eq:six_tendon_mapping} by \(\mat C\trans\) therefore
gives the exact inverse on the pure-bending subspace:
\begin{equation}
  \vect c=\frac{1}{3dL}\mat C\trans\vect q .
  \label{eq:supp_inverse}
\end{equation}
Its two entries are
\begin{align}
  c_x&=
  \frac{
  q_1+\tfrac12q_2-\tfrac12q_3-q_4-\tfrac12q_5+\tfrac12q_6
  }{3dL},
  \label{eq:supp_inverse_cx}\\
  c_y&=
  \frac{\tfrac{\sqrt3}{2}(q_2+q_3-q_5-q_6)}{3dL}.
  \label{eq:supp_inverse_cy}
\end{align}
When the three opposite-pair relations hold, these reduce to
\begin{equation}
  c_x=\frac{2q_1+q_2-q_3}{3dL},
  \quad
  c_y=\frac{q_2+q_3}{\sqrt3\,dL}.
  \label{eq:supp_pair_inverse}
\end{equation}
The factor \(1/(3dL)\) is thus analytical, not a fitted pseudoinverse.

\section{Null space, compatibility, and orthogonal projectors}
\label{sec:supp_null}
Because \(\rank(\mat C)=2\), the six-dimensional actuation space decomposes
as
\begin{equation}
    \R^6=\operatorname{range}(\mat C)
  \mathbin{\oplus}\nullsp(\mat C\trans).
\end{equation}
The orthogonal projectors onto these subspaces are
\begin{equation}
  \mat P_R=\frac13\mat C\mat C\trans,\quad
  \mat P_N=\mat I_6-\frac13\mat C\mat C\trans.
  \label{eq:supp_projectors}
\end{equation}
Using \(\mat C\trans\mat C=3\mat I_2\),
\begin{equation}
    \mat P_R^2
  =\frac19\mat C(\mat C\trans\mat C)\mat C\trans
  =\mat P_R,\quad
  \mat P_R\trans=\mat P_R.
\end{equation}
It follows that \(\mat P_N^2=\mat P_N\),
\(\mat P_N\trans=\mat P_N\),
\(\mat P_R\mat P_N=\mat0\), and
\(\mat C\trans\mat P_N=\mat0\).

One explicit basis of \(\nullsp(\mat C\trans)\) is
\begin{equation}
  \mat N=
  \begin{bmatrix}
    1&0&0&1\\
    0&1&0&-1\\
    0&0&1&1\\
    1&0&0&-1\\
    0&1&0&1\\
    0&0&1&-1
  \end{bmatrix},
  \quad
  \mat C\trans\mat N=\mat0.
  \label{eq:supp_null_basis}
\end{equation}
The first three columns represent equal common-mode changes within each antagonistic tendon pair. The fourth column extends the three-pair vector $\vect n_a=[1,-1,1]\trans$ to the full six-tendon coordinate space. All four directions are algebraically shape-neutral under the two-component curvature model. However, shape neutrality in this model does not imply physical feasibility of the corresponding tendon motions.

For an arbitrary measured vector \(\vect q_m\),
\begin{equation}
  \vect q_m=\mat P_R\vect q_m+\mat P_N\vect q_m,\quad
  \hat{\vect c}=\frac{1}{3dL}\mat C\trans\vect q_m,
  \quad
  \vect r_q=\mat P_N\vect q_m .
  \label{eq:supp_decomposition}
\end{equation}
The representable part is
\(\mat P_R\vect q_m=dL\mat C\hat{\vect c}\), while
\(\vect r_q\) is rejected by the ideal curvature estimator.
An in-range bias \(dL\mat C\vect b_c\) has zero residual and produces the
curvature bias \(\vect b_c\), which is why a small residual is not a complete
model-validation test.

All tendon velocities that give the ideal reference curvature velocity
\(\dot{\vect c}_r\) are
\begin{equation}
  \dot{\vect q}
  =dL\mat C\dot{\vect c}_r+\mat P_N\vect z,
  \quad \vect z\in\R^6.
  \label{eq:supp_general_allocation}
\end{equation}
The two terms are orthogonal. Therefore,
\[
  \norm{\dot{\vect q}}_2^2
  =3d^2L^2\norm{\dot{\vect c}_r}_2^2
   +\norm{\mat P_N\vect z}_2^2,
\]
so setting \(\vect z=\vect0\) is the unique minimum-Euclidean-norm
velocity allocation. The statement concerns the algebraic velocity norm and
does not imply minimum energy, positive tension, or slack prevention.

\section{Explicit tendon and reel commands}
\label{sec:supp_commands}

Define
\begin{align}
  v_x&=\dot c_{d,x}-\frac{\dot B_x}{B_x}e_x+k_xe_x,
  \label{eq:supp_vx}\\
  v_y&=\dot c_{d,y}-\frac{\dot B_y}{B_y}e_y+k_ye_y.
  \label{eq:supp_vy}
\end{align}
Then \(\dot{\vect c}_r=[v_x,v_y]\trans\), and the six entries of
\(\dot{\vect q}=dL\mat C\dot{\vect c}_r\) are
\begin{align}
  \dot q_1&=dLv_x,\nonumber\\
  \dot q_2&=dL\left(\tfrac12v_x+\tfrac{\sqrt3}{2}v_y\right),\nonumber\\
  \dot q_3&=dL\left(-\tfrac12v_x+\tfrac{\sqrt3}{2}v_y\right),\nonumber\\
  \dot q_4&=-dLv_x,\nonumber\\
  \dot q_5&=dL\left(-\tfrac12v_x-\tfrac{\sqrt3}{2}v_y\right),\nonumber\\
  \dot q_6&=dL\left(\tfrac12v_x-\tfrac{\sqrt3}{2}v_y\right).
  \label{eq:supp_six_commands}
\end{align}
They obey
\begin{equation}
    \dot q_4=-\dot q_1,\quad
  \dot q_5=-\dot q_2,\quad
  \dot q_6=-\dot q_3,\quad
  \dot q_1-\dot q_2+\dot q_3=0.
\end{equation}
The total and component rate bounds follow from
\(\mat C\trans\mat C=3\mat I_2\):
\begin{equation}
  \norm{\dot{\vect q}}_2
  =\sqrt3\,dL\norm{\dot{\vect c}_r}_2,
  \quad
  \abs{\dot q_i}\le dL\norm{\dot{\vect c}_r}_2.
  \label{eq:supp_rate_bounds}
\end{equation}

Because \(q_i=l_{i,\mathrm{ref}}-l_i\) and the references are constant,
\begin{equation}
  \dot l_i=-\dot q_i.
  \label{eq:supp_length_sign}
\end{equation}
If reel \(i\) has calibrated effective radius \(r_{s,i}>0\) and measured
winding sign \(\sigma_i\in\{-1,1\}\), a constant-radius interface is
\begin{equation}
  \omega_i=\sigma_i\frac{\dot q_i}{r_{s,i}}.
  \label{eq:supp_motor_speed}
\end{equation}
Both \(r_{s,i}\) and \(\sigma_i\) are calibration quantities. Layered
winding can make the effective radius state dependent, in which case
Eq.~\eqref{eq:supp_motor_speed} must be replaced by a measured nonlinear
conversion. None of these relations enforces positive tendon tension.

\FloatBarrier

\section*{CRediT authorship contribution statement}
\addcontentsline{toc}{section}{CRediT authorship contribution statement}
\textbf{Yi Lu:} Writing– original draft, Methodology. \textbf{Chao Tang:} Writing– review \& editing. \textbf{Zhiji Han:} Writing– original draft, Methodology. \textbf{Hongdu Wang:} Writing– review \& editing.

\section*{Declaration of competing interest}
\addcontentsline{toc}{section}{Declaration of competing interest}
The authors declare that they have no known competing financial interests or personal relationships that could have appeared to influence the work reported in this paper.

\section*{Acknowledgments}
\addcontentsline{toc}{section}{Acknowledgments}
This work was partially supported by the National Natural Science Foundation of China under Grant No. 62503443; the Young Talent of Lifting Engineering for Science and Technology in Shandong under Grant No. SDAST2026QTA092; the Postdoctoral Innovation Program of Shandong Province under Grant No. SDCX-ZG-202501023; the Qingdao Natural Science Foundation under Grant No. 25-1-1-116-zyyd-jch; the Qingdao Postdoctoral Program under Grant No. QDBSH20250102058; and the Fundamental Research Funds for the Central Universities under Grant No. 202513025.

\section*{Data availability}
\addcontentsline{toc}{section}{Data availability}
No data was used for the research described in the article.

\pdfbookmark{References}{}
\bibliographystyle{elsarticle-num-nodoi}
\bibliography{references}

\end{document}